\documentclass{article}

\usepackage{amsmath,amssymb,pifont}
\usepackage{multicol}
\usepackage{amstext}
\usepackage{amsthm}
\usepackage{multirow}
\usepackage{booktabs}
\usepackage{times}
\usepackage{lipsum}
\usepackage[shortlabels]{enumitem}
\usepackage{cancel}
\usepackage{wrapfig}
\usepackage{array}
\usepackage{siunitx}
\usepackage{csvsimple}
\usepackage[multidot]{grffile}
\usepackage{bbm}
\usepackage{dblfloatfix}
\usepackage{hyperref}
\usepackage{makecell}
\usepackage{bbm, dsfont}
\usepackage{mathtools}
\usepackage[dvipsnames]{xcolor}
\usepackage{comment}
\usepackage[capitalise]{cleveref}

\usepackage{geometry}
\usepackage{mathpazo}
\usepackage{enumitem}

\usepackage[multiple]{footmisc}
\usepackage{mathrsfs}
\usepackage{todonotes}
\usepackage{tikz}
\usepackage{hyperref}
\usepackage{cleveref}
\usepackage{algpseudocode,algorithm,algorithmicx}
\usepackage{xfrac}

\usepackage{graphicx} 
\usepackage{color}

\usepackage{array}
\usepackage{amssymb}
\usepackage{amsmath}
\usepackage{xspace}
\usepackage{fancyhdr}
\usepackage{comment}

\hypersetup{final}

\newcommand{\eps}{\ensuremath{\varepsilon}}

\newcommand{\calA}{\ensuremath{\mathcal{A}}}

\newcommand{\calC}{\ensuremath{\mathcal{C}}}
\newcommand{\calD}{\ensuremath{\mathcal{D}}}

\newcommand{\calL}{\ensuremath{\mathcal{L}}}

\newcommand{\calT}{\ensuremath{\mathcal{T}}}

\renewcommand{\Pr}{\mathop{\mathbf{Pr}}}
\newtheorem{lem}{Lemma}[section]

\newtheorem{thm}[lem]{Theorem}

\newtheorem{definition}[lem]{Definition}

\DeclareMathOperator*{\argmin}{arg\,min}

\makeatletter
\newcommand{\vast}{\bBigg@{4}}
\newcommand{\Vast}{\bBigg@{5}}
\makeatother

\newcommand{\ex}[2]{{\ifx&#1& \mathbb{E} \else
\underset{#1}{\mathbb{E}} \fi \left[#2\right]}}
\newcommand{\pr}[2]{{\ifx&#1& \mathbb{P} \else
\underset{#1}{\mathbb{P}} \fi \left[#2\right]}}

\newcommand{\ltwo}[1]{\left\|#1\right\|_2}

\usepackage{ulem}

\DeclarePairedDelimiterX{\infdivx}[2]{(}{)}{%
  #1\;\delimsize\|\;#2%
}

\newcommand{\mypar}[1]{\smallskip
	\noindent{\textbf{{#1}:}}}
	
\renewcommand{\epsilon}{\varepsilon}

\renewcommand{\tilde}{\widetilde}

\newtheorem{theorem}[lem]{Theorem}

\newcommand{\Expect}[2]{\mathbb{E}_{#1}\left[#2\right]}

\usepackage{microtype}
\usepackage{graphicx}
\usepackage{subfigure}
\usepackage{booktabs} % for professional tables

\usepackage{fullpage}
\usepackage[numbers, sort, comma, square]{natbib}

\author{Arun Ganesh\thanks{Google \texttt{\{arunganesh,srxzr,sewoongo,steinke,omthkkr,athakurta,lunwang\}@google.com}},\hspace{0.3cm} Mahdi Haghifam\thanks{University of Toronto; Part of this work was done while the author was an intern at Google Brain. \texttt{mahdi.haghifam@mail.utoronto.ca}}, \hspace{0.3cm} Milad Nasr\footnotemark[1], \hspace{0.3cm} Sewoong Oh\footnotemark[1]\textsuperscript{ ,}\thanks{University of Washington}, \\ Thomas Steinke\footnotemark[1], \hspace{0.3cm}Om Thakkar\footnotemark[1], \hspace{0.3cm}Abhradeep  Thakurta\footnotemark[1], \hspace{0.3cm}  Lun Wang\footnotemark[1]}
\date{}

\begin{document}

\title{Why Is Public Pretraining Necessary for Private Model Training?}
\maketitle

\begin{abstract}
In the privacy-utility tradeoff of a model trained on benchmark language and vision tasks, remarkable improvements have been widely reported with the use of pretraining on publicly available data. This is in part due to the benefits of transfer learning, which is the standard motivation for pretraining in non-private settings. 
However, the stark contrast in the improvement achieved through pretraining under privacy compared to non-private settings suggests that there may be a deeper, distinct cause driving these gains. 
To explain this phenomenon, we hypothesize that the non-convex loss landscape of a model training necessitates an optimization algorithm to go through two phases. In the first, the algorithm needs to select a good ``basin'' in the loss landscape. In the second, the algorithm solves an easy optimization within that basin. The former is a harder problem to solve with private data, while the latter is harder to solve with public data due to a distribution shift or data scarcity. Guided by this intuition, we provide theoretical constructions that provably demonstrate the separation between private training with and without public pretraining. 
Further, systematic experiments on CIFAR10 and LibriSpeech  provide supporting evidence for our hypothesis.
\end{abstract}

\section{Introduction}\label{sec:intro}

As modern machine learning models are increasingly capable of memorizing the training data, membership inference attacks and data reconstruction attacks have successfully demonstrated the vulnerability of sharing models trained on sensitive data. 
Differential Privacy (DP), introduced in \cite{DMNS}, is now a gold standard  measure of privacy leakage in training a model, which is parameterized by two scalars:   $\varepsilon>0$ and $\delta\in[0,1]$. By introducing enough randomness in the training, one can ensure that the model does not depend too much on each individual training example. This provides plausible deniability to the participants \cite{KOV17} and evades privacy attacks, achieving  strong DP with small values of $(\varepsilon,\delta)$. We give a formal definition in Definition~\ref{def:DP}.

One of the main challenges in training on private data is that 
utility and privacy trades off unfavorably on standard benchmark tasks.  Given a target task, such as table-to-text generation, on a private dataset, say E2E dataset \citep{novikova2017e2e},  state-of-the-art techniques suffer from significant performance degradation to achieve even an acceptable level of privacy. For example, a weak privacy guarantee of $\varepsilon=8$ significantly deteriorates the performance of the trained model compared to the one trained without privacy, i.e.~$\varepsilon=\infty$ (second row of Table~\ref{tab:1}). 
Perhaps surprisingly, there is one simple change to the training algorithm that can significantly reduce this cost of privacy:  pretraining the model on some public data (first row of Table~\ref{tab:1}).

\begin{table}[h]
  \begin{center}
        \begin{tabular}{|l| r r| r|} 
        \hline
   & & & cost of\;  \\ 
   & $\eps=\infty$ & $\eps=8$ & privacy\\
   %$\eps=3$ \\ [0.5ex] 
 \hline 
 with public pretrain & 69.46 & 63.19 & 6.27 \\ 
 without public pretrain & 65.73 & 24.25& 41.48 \\
 \hline 
  gain of public pretraining & 3.73 & 38.94 & \\
 \hline
    \end{tabular}
    \caption{BLEU score for generating  descriptions of table entries on E2E dataset reported in \citep[Table 2]{li2021large} with $\delta=10^{-5}$. The first row uses  GPT-2 \cite{radford2019language} as a pretrained model. }
    \label{tab:1}
    \end{center}
\end{table}

Such remarkable gain of public pretraining has been widely observed in standard benchmark vision and language tasks, which we survey in Appendix~\ref{app:survey}. This includes CIFAR-10, MNIST, and Fashion MNIST in \cite{tramer2020differentially},  CIFAR-100, ImageNet, and Places-365 in \cite{de2022unlocking}, text generation with E2E and DART in \citep{li2021large}, and next word prediction on Reddit dataset  \cite{kerrigan2020differentially}. 
Note that in all these cases, the  public data distribution differs from the target task distribution. Nevertheless, we expect some gain from public pretraining, drawing analogy from its  success in non-private training of large models (e.g., first column in Table~\ref{tab:1}). However, the stark difference in the gain of pretraining between the non-private case, i.e., $\varepsilon=\infty$, and the weakly private case, say  $\varepsilon=8$, is striking. This suggests that the benefit of public pretraining in differentially private machine learning is a fundamentally different phenomenon from the typical benefits of standard transfer learning \citep{bozinovski1976influence,sharif2014cnn,bommasani2021opportunities}. Our goal is to give an insight into when such a phenomenon can be observed  by carefully constructing synthetic public and private tasks. Recently, in a closely related work,~\cite{li2022does} formally demonstrated that public data mitigates the curse of dimensionality when fine-tuning with privacy. However, to the best of our knowledge, ours is the first work to understand the {\it necessity} of public data in private model training.

\begin{figure}[h!]
    \subfigure[]{\includegraphics[width=0.4\textwidth]{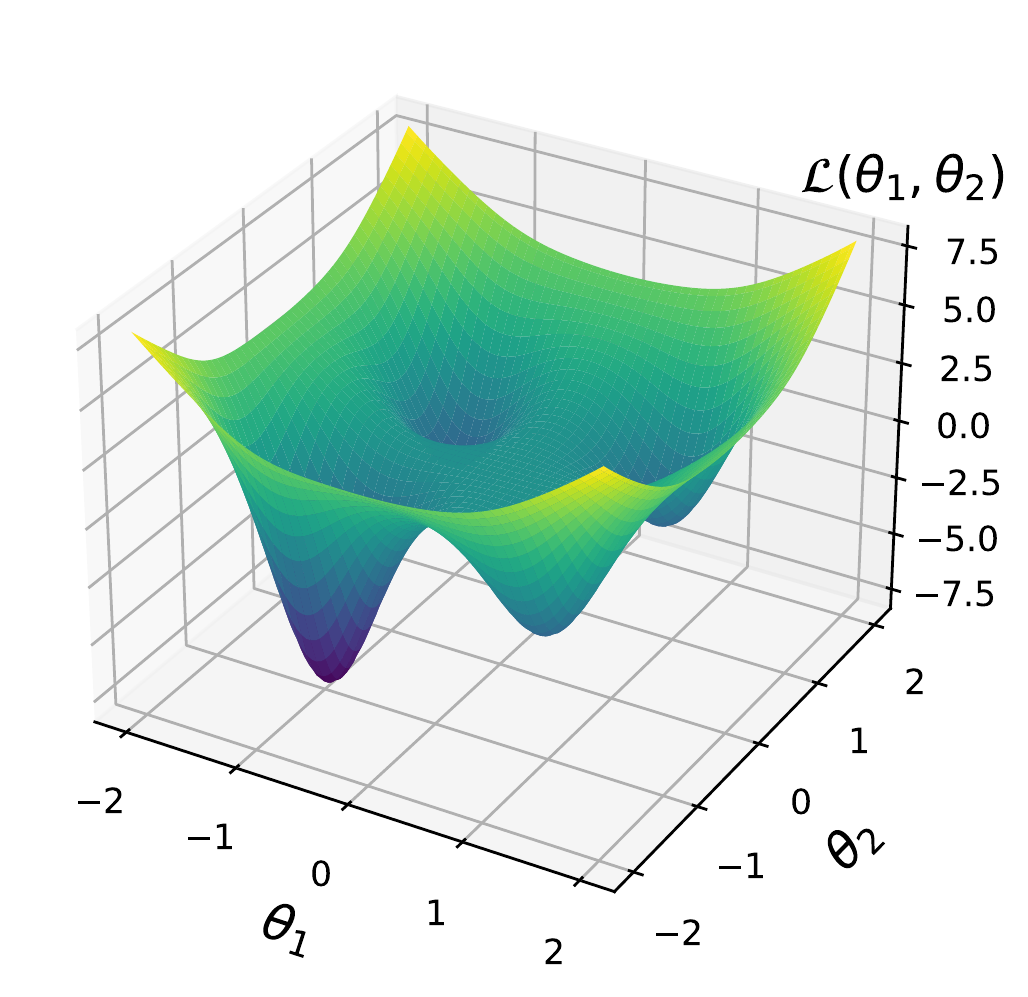}}
    \subfigure[]{\includegraphics[width=0.55\textwidth]{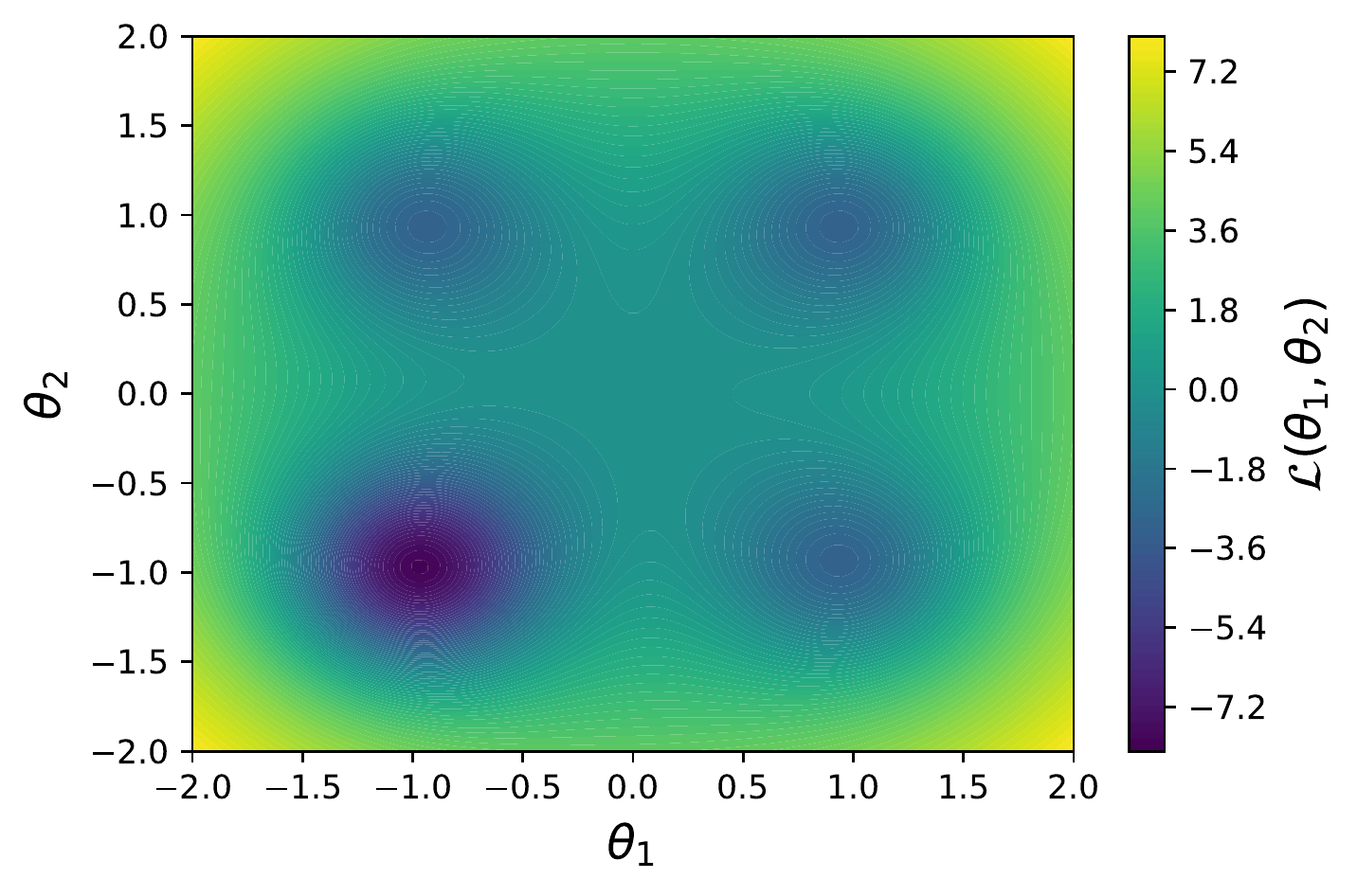}}
    \caption{An example of a non-convex loss function. 
    While the overall function is non-convex, it consists of  many locally convex ``basins'', some better than the others.
    }
    \label{fig:cubeconstruction}
\end{figure}

In this paper, we provide a theoretical example of a loss function that~\textit{requires} pretraining on public data and fine-tuning with private data. Our construction is guided by our hypothesis that the typical population loss landscape of standard machine learning tasks necessitates gradient based algorithms to go through two stages.  A conceptual two-dimensional sketch of the landscape we envision is shown in \cref{fig:cubeconstruction}.  We start from a random initialization close to the origin.  In the first stage, the algorithm is directed by the data towards a good basin with small local minima.  This is followed by the second stage, where the algorithm solves what is effectively a convex optimization in the selected basin to arrive at the local minima.  The key insight is that the first stage of selection should require significantly more samples to solve privately, compared to the number of samples required to solve it without privacy. Concretely, for the example in \cref{fig:cubeconstruction}, the gradient at the origin directs to the correct basin containing the global minima, but the gradient is small.  A private gradient descent adds additional noise to the update, increasing the chance of ending up at worse basins.  Hence, a significantly larger  private dataset is needed to overcome the privacy noise. 
This construction is motivated by the private hypothesis selection problems where a similar  fundamental separation in sample complexity is known  \cite{steinke2017tight}. This intuition would explain the widely observed failure of private training when starting from a random initialization. We turn this hypothesis into concrete constructions in Section~\ref{sec:construct}, where we formally prove the separation in sample complexity.  

\mypar{Main contributions} 
In Section~\ref{sec:construct}, we construct theoretical tasks to demonstrate the fundamental separation in sample complexity. First, we construct a theoretical loss function and a corresponding data distribution such that given $n_{pub}$ public samples and $n_{priv}$ private samples from this distribution with $n_{pub} \ll n_{priv}$, pretraining on the public data and fine tuning with the private data achieves a much better loss than any algorithm with access to either alone. Next, we extend our result to a more relevant setting where $n_{pub}$ is large but the public data is out of distribution.
This construction exhibits the need to have little to no privacy noise in the first ``phase'' of non-convex optimization. To the best of our knowledge, this is the first theoretical analysis demonstrating the need for public pretraining.

In Section~\ref{sec:exp}, we empirically validate our two-phase hypothesis. First, treating CIFAR-10 as our target private task, we consider a setup where we are allowed $T$ epochs of pre- or post-training on in-distribution public data, out-of-distribution public data, or private data with low noise. In all settings we demonstrate it is best to use all these low- or non-private training epochs on pretraining (as opposed to post-training). This demonstrates that early rounds of training are more sensitive to privacy noise, as conjectured in our two-phase hypothesis. Secondly, we look at a manifold of the loss landscape interpolated between three models trained on LibriSpeech. We show that a publicly pretrained and privately fine-tuned model ends up in the same basin as a fully publicly trained model. On the other hand, a fully privately trained model ends up in a different basin. This provides evidence that public pretraining's benefits are in part due to selecting a better basin for fine-tuning.

\subsection{Other Related Work}

Pretraining on  public data is now a default choice in large scale private training for NLP tasks \cite{yu2021differentially, he2022exploring, bu2022differentially, ginart2022submix}, including 175 billion parameter GPT-3 with $\varepsilon=1$, and vision tasks \cite{golatkar2022mixed, luo2021scalable, kurakin2022toward, bu2022differentially2, de2022unlocking}. Motivated by pretraining providing good feature representations, \cite{tramer2020differentially} propose using handcrafted features for small scale problems, as opposed to learned features, to improve utility-privacy tradeoff.
On the other hand, \citep{tramer2022considerations} cautions against the indiscriminate use of large-scale public data in DP training, which we discuss in depth in Section~\ref{sec:discuss}.

Besides the aforementioned empirical results, public data has been used to show theoretical improvements for problems such as query release \cite{alon2019limits,BassilyCMNUW20,liu2021leveraging}, mean estimation \cite{avent2020hybrid,bie2022estimation}, and optimization \cite{zhou2020bypassing, kairouz2020fast, asi2021adapting, amid2022mirrordescent}. In the optimization case, besides pretraining, these papers use public data to learn the geometry of the private loss in various ways and use geometry-aware gradient descent methods, rather than vanilla DP-SGD.

\cite{steinke2017tight} showed that for the problem of selecting  the $k$ coins out of $d$ coins that land heads with the highest probability, any $(\epsilon, \delta)$-DP algorithm with constant error requires $n = \Omega(\sqrt{k} \log d)$ samples from each coin. This is in contrast with the non-private case, where $n = O(\log d)$ suffices for any $k$. Selection and non-convex optimization are tightly connected: \cite{GTU22} show a reduction from selection to non-convex optimization, by designing a loss with $d$ locally convex basins, each corresponding to a different coin in the selection problem. This gives a different perspective on why the first stage of non-convex optimization may be difficult privately but not with public data: it effectively involves solving a selection problem on the basins in the loss function.

\subsection{Background on differential privacy and DP-SCO} 

Differential privacy is a privacy guarantee for algorithms that can be viewed as random functions of datasets:

\begin{definition}[{Differential Privacy \cite{DMNS}}]\label{def:DP} 
Let $\calD$ be a data domain, and $\calC$ be a set of outputs. An algorithm $\calA: \calD^* \rightarrow \calC$ is $(\epsilon, \delta)$-differentially private if for any $D, D' \in \calD^*$ such that $D$ and $D'$ differ in at most one element and any set of outputs $S \subseteq \calC$: $   \Pr_{\theta \sim \calA(D)}\left[\theta \in S\right] \;\;\leq\;\; e^\epsilon \Pr_{\theta \sim \calA(D')}\left[\theta \in S\right] + \delta$.
\end{definition}

A well-studied problem in the differential privacy literature is differentially private stochastic (convex) optimization (DP-SCO) \cite{BST14, BassilyFTT19, feldman2019private, bassily2020stability,kulkarni2021private, asi2021adapting, gopi2022private}. In DP-SCO, there is a loss function $\ell: \calC \times \calD \rightarrow \mathbb{R}$, and an unknown distribution $\tau$ over $\calD$. Given $n$ i.i.d. samples from $\tau$, we wish to find $\theta \in \calC$ minimizing the population loss $\calL(\theta) := \Expect{d \sim \tau}{\ell(\theta; d)}$. For any $\tau$ we denote the population minimizer by $\theta^*(\tau) := \argmin_{\theta \in \calC} \calL(\theta)$. The performance of a DP-SCO algorithm is measured by its \textit{risk}, $\Expect{D \sim \tau^n, \theta \sim \calA(D)}{\calL(\theta)} - \calL(\theta^*(\tau))$. DP-SCO captures most machine learning tasks we are interested in. The most widely studied algorithm in the DP-SCO literature is DP-SGD \cite{song2013stochastic, BST14, DP-DL, BassilyFTT19, bassily2020stability}, which minimizes the empirical loss $\ell(\theta; D) = (1/|D|) \sum_{d \in D} \ell(\theta; d)$ over $\calC \subseteq \mathbb{R}^p$ as follows: DP-SGD starts with $\theta_0$, and for $t$ iterations computes $\theta_{t+1} = \theta_t - \eta_t \nabla \ell(\theta_t; D) + \xi_t,$
where $\xi_t\sim N(0, \sigma^2 \mathbb{I})$ and $\sigma^2$ is chosen to satisfy ($\varepsilon,\delta$)-DP. 

Perhaps the simplest problem captured by DP-SCO is private mean estimation with identity covariance. The following lemma gives a lower bound on private mean estimation. It follows from Theorem 5.5 of \cite{BST14} and standard translation of ERM lower bounds to SCO lower bounds (see Appendix C of \cite{BassilyFTT19}):

\begin{lem}\label{lem:privatesclb}
For $\ell(\theta; d) = (1/2) \ltwo{\theta - d}^2$, $\calC = \mathbb{R}^p$, and $\calD = B_{p}(0, 1)$ (the $p$ dimensional $\ell_2$-ball of radius 1 centered at the origin), let $\theta^*(\tau) := \argmin_{\theta \in \calC} \calL(\theta)$ for a distribution $\tau$ over $\calD$. For $p \leq \epsilon^2 n^2$ and $\delta = o(1/n)$, there exists a set of distributions, $\calT_1$, over $\calD$, such that the following is true. For every $(\epsilon, \delta)$-DP algorithm $\calA: \calD^n \rightarrow \calC$, there exists $\tau(\calA) \in \calT_1$ such that:
\begin{align*}
&\Expect{D \sim \tau(\calA)^n, \theta \sim \calA(D)}{\calL(\theta)} =\calL(\theta^*(\tau(\calA))) + \Omega\left(\frac{p}{\epsilon^2 n^2} + \frac{1}{n}\right).
\end{align*}
Furthermore, for some $M = \Omega(\frac{\sqrt{p}}{\epsilon n})$ and all such $\tau \in \calT_1$, $|\ltwo{\theta^*(\tau)} - M| \leq 1/n$.
\end{lem}
Non-privately, this translates to:
\begin{lem}\label{lem:publicsclb}
For $\ell(\theta; d) = \frac{1}{2}\ltwo{\theta - d}^2$, $\calC = \mathbb{R}^p$, and $\calD = B_{p}(0, 1)$, there exists a set of distributions, $\calT_2$, over $\calD$ such that the following is true. For every $\calA: \calD^n \rightarrow \calC$ there exists $\tau(\calA) \in \calT_2$ such that: 
\[\Expect{D \sim \tau(\calA)^n, \theta \sim \calA(D)}{\calL(\theta)} =\calL(\theta^*(\tau(\calA))) +\Omega\left(\frac{1}{n}\right).\]
\end{lem}

These lemmas are the basis of the results in Section~\ref{sec:construct}.
Results in \cite{BST14} and standard translations from empirical loss bounds to population loss bounds via uniform stability (see e.g. \cite{HardtRS16}) show that DP-SGD achieves upper bounds for mean estimation that match these lower bounds up to polylogarithmic factors.
\section{Necessity of public pretraining} 
\label{sec:construct} 

A typical scenario in pretraining on public data is when the public dataset is large but is Out-Of-Distribution (OOD); there is a potentially large distribution shift between the  public  and the  private dataset \cite{yu2021differentially, he2022exploring, bu2022differentially, ginart2022submix,golatkar2022mixed, luo2021scalable, kurakin2022toward, bu2022differentially2, de2022unlocking}.  
In this section, we start with a simpler scenario where a small number of In-Distribution (ID) samples are used in public pretraining.  This simplifies the explanation of our construction and also  corresponds to realistic scenarios where public data comes from users who consented. The more common OOD case is addressed in Section~\ref{sec:ood}. 
\subsection{Pretraining on in-distribution public data} 
\label{sec:id}

When a small number of in-distribution samples are publicly available, several techniques have been proposed to improve the accuracy-privacy trade-off. An immediate use is to reduce the sensitivity of a mini-batch gradient by including the public data in the mini-batch. The public data can also be used to compute useful statistics; one can reduce the privacy noise by projecting the gradient onto a low-dimensional subspace computed from public data  \cite{kairouz2020fast,yu2021not,zhou2020bypassing,golatkar2022mixed} and by improving the adaptive clipping method with the geometry of the gradients estimated from public data  \cite{golatkar2022mixed,asi2021private,dopesgd}. 
However, by far the most dominant technique in terms of the accuracy gain is pretraining on the in-distribution public data. For example, on CIFAR-10 dataset, one can train a $(\varepsilon=2,\delta=10^{-5})$-DP model that achieves  64.9\% test  accuracy. Treating 4\% of the training dataset as public data, the accuracy can be improved by 7.1\% \citep[Table 1]{dopesgd}. All the other techniques only give 2.8\% extra gain, which includes using public data in fine-tuning, public data assisted adaptive clipping, and averaging past iterates.  Such pretraining with in-distribution public data has been  successful also in training variational autoencoders \cite{jiang2022dp}.  We provide systematic study of these gains with numerical experiments on benchmark datasets in \cref{sec:exp}.  %Figure~\ref{fig:pre_post_public}. 

Motivated by the practical successes, we first consider the following setup. We are given $n_{pub}$ public examples, $D_{pub}$, and $n_{priv}$ private examples, $D_{priv}$, both drawn i.i.d. from the same distribution $\tau$, where  $n_{pub}\ll n_{priv}$. We construct $\tau$ such that pretraining on small ID public data can significantly improve the performance of a private training. Concretely, we will show that for any integer $p$, there exists a loss function $\ell$, sample sizes $n_{pub}$ and $n_{priv}$, and a data distribution $\tau$ such that $(i)$ any non-private algorithm $\calA_{pub}$ given only $D_{pub}$ has worst-case excess population loss lower bounded by $\Omega(1)$; $(ii)$ any $(\epsilon, \delta)$-DP algorithm $\calA_{priv}$ given only $D_{priv}$ has worst-case excess population loss lower bounded by $\Omega(1)$; and $(iii)$ a gradient-based algorithm $\calA_{mixed}$ that pretrains on $D_{pub}$ and privately fine-tunes on $D_{priv}$ achieves excess population loss upper bounded by $O(1/p)$. In particular, the dimensionality of $\ell$, $n_{pub}$, and $n_{priv}$ are polynomial functions of $p$. We focus on the unconstrained case where $\calC = \mathbb{R}^p$, as it aligns with how differentially private learning models are trained in practice.

% ------------------------------------------------
\subsection{Construction} 
\label{sec:id_construction} 

We  first give a high-level overview of a construction for our main theorem and defer details to Appendix~\ref{app:wells}. While our construction builds on upper/lower bounds for public/private mean estimation, one can build a similar construction using upper/lower bounds for linear regression instead. This follows via standard reductions from mean estimation to linear regression. We focus here on mean estimation for simplicity of presentation. 
A reference for notation  is in Appendix~\ref{app:notation}.

%{\bf Todo: Explain reduction to mean estimation from linear regression.} 
Our strategy is to concatenate the two known lower bounds for mean estimation with private data in Lemma~\ref{lem:privatesclb} and with public data in Lemma~\ref{lem:publicsclb}.  We consider a distribution $\tau$ over a data point $d = (d_1,d_2) \in\mathbb{R}^{p^4}\times \mathbb{R}^p$ whose  population mean is $ \theta^*(\tau)=(\theta^*_1(\tau), \theta^*_2(\tau) ) \in \mathbb{R}^{p^4}\times \mathbb{R}^p$. The first $p^4$ coordinates are used to construct a hard distribution for private mean estimation with a loss function 
$\ell_1:{\mathbb R}^{p^4}\times{\mathbb R}^{p^4}\to{\mathbb R} $, and the following $p$ coordinates are used to construct a hard distribution for public mean estimation with a loss function $\ell_2:{\mathbb R}^{p}\times{\mathbb R}^{p}\to{\mathbb R} $. 
We assume we have $n_{pub}$ public samples and $n_{priv}$ private samples from the same distribution with $n_{pub} \ll n_{priv}$.

We will define an appropriately chosen basin $S\subset {\mathbb R}^{p^4}$ and eventually combine our loss functions in Eq.~\eqref{def:S} such that if $\theta_1$ is   far from $S$, then $\ell((\theta_1, \theta_2)) =\ell_1(\theta_1)$, but inside of $S$, $\ell((\theta_1, \theta_2)) = \ell_1(\theta_1)+\ell_2(\theta_2)$. In particular, we will choose $\ell_2$ that is non-positive everywhere, so that is desirable to be in $S$ with respect to minimizing $\ell$.

%\agnoteinline{@Sewoong, modified the below paragraph pretty heavily; the private loss can reach S, it just can't always reach a good point in S.} Looks good to me!
Starting outside of $S$, the algorithm first needs to minimize $\ell_1$ to reach $S$. We use $\ell_1$ from the private lower bound (Lemma~\ref{lem:privatesclb}) such that a private algorithm fails just on optimizing $\ell_1$. On the other hand, an algorithm with a small amount of public data can easily optimize $\ell_1$. We will eventually choose $S$ that contains all points close to the optimum of $\ell_1$, so any public algorithm will reach $S$ after optimizing $\ell_1$, and will not touch $\theta_2$ in doing so. Once inside the basin $S$, the algorithm needs to also minimize $\ell_2$ to reach a small total loss. We use $\ell_2$ from the public lower bound (Lemma~\ref{lem:publicsclb}) such that a small-size public data alone is not sufficient to (approximately) reach global minima but large-size private data can. Precisely, we combine the two loss functions and define
\begin{equation}
\ell((\theta_1, \theta_2); (d_1, d_2)) = \ell_1(\theta_1; d_1) + p \, q(\theta_1) \cdot \ell_2(\theta_2; d_2), \label{def:loss} 
\end{equation}
where 
\begin{equation*}
q(\theta_1) := \left\{\begin{array}{ll}
0, &  \ltwo{\theta_1 - \Pi_S(\theta_1)} > R_2\\
1 - \frac{\ltwo{\theta_1 - \Pi_S(\theta_1)}}{R_2}, & 0 < \ltwo{\theta_1 - \Pi_S(\theta_1)} \leq R_2\\
1 & \ltwo{\theta_1 - \Pi_S(\theta_1)} = 0 \text{ }(\text{i.e., } \theta_1 \in S)
\end{array}
\right. \;,
\end{equation*} 
for some $S\subset {\mathbb R}^{p^4}$  and 
$R_2>0$ to be defined later. Here $\Pi_S$ denotes Euclidean projection into $S$.
If $\theta_1$ is far from $S$, $\ell$ is just $\ell_1(\theta_1)$. If $\theta_1$ is in $S$, then $\ell$ is just $\ell_1(\theta_1) + p \cdot \ell_2(\theta_2)$. In between these two regimes, $\ell$ interpolates between these two loss functions; this interpolation is technically not necessary for our eventual theorem and proof, but gives a more realistic loss function. Note that $\ell_2$ is non-positive, so having larger $q(\theta_1)$ (i.e., being in or close to $S$) is advantageous with respect to minimizing the term depending on $\theta_2$.

%{\bf Sewoong: Give a simple high-level explanation of the figure. We do not have all the definitions to properly explain.} 

In Figure~\ref{fig:wells} is an example of our eventual construction. $S$ consists of two basins, centered at $-0.5$ and $0.5$. If $\theta_1$ is near one of these points, then $\ell$ is a quadratic centered at $.005$ with respect to $\theta_2$. If $\theta_2$ is far from these points, $\ell$ is a constant with respect to $\theta_2$. So, if we start at the origin, using gradient-based methods we would first have to optimize $\theta_1$ to get to one of the basins, and then optimize $\theta_2$. With private data choosing the right basin is hard, with public data optimizing $\theta_2$ within a basin is hard.

\begin{figure}[h!]
    \subfigure[]{\includegraphics[width=0.4\textwidth]{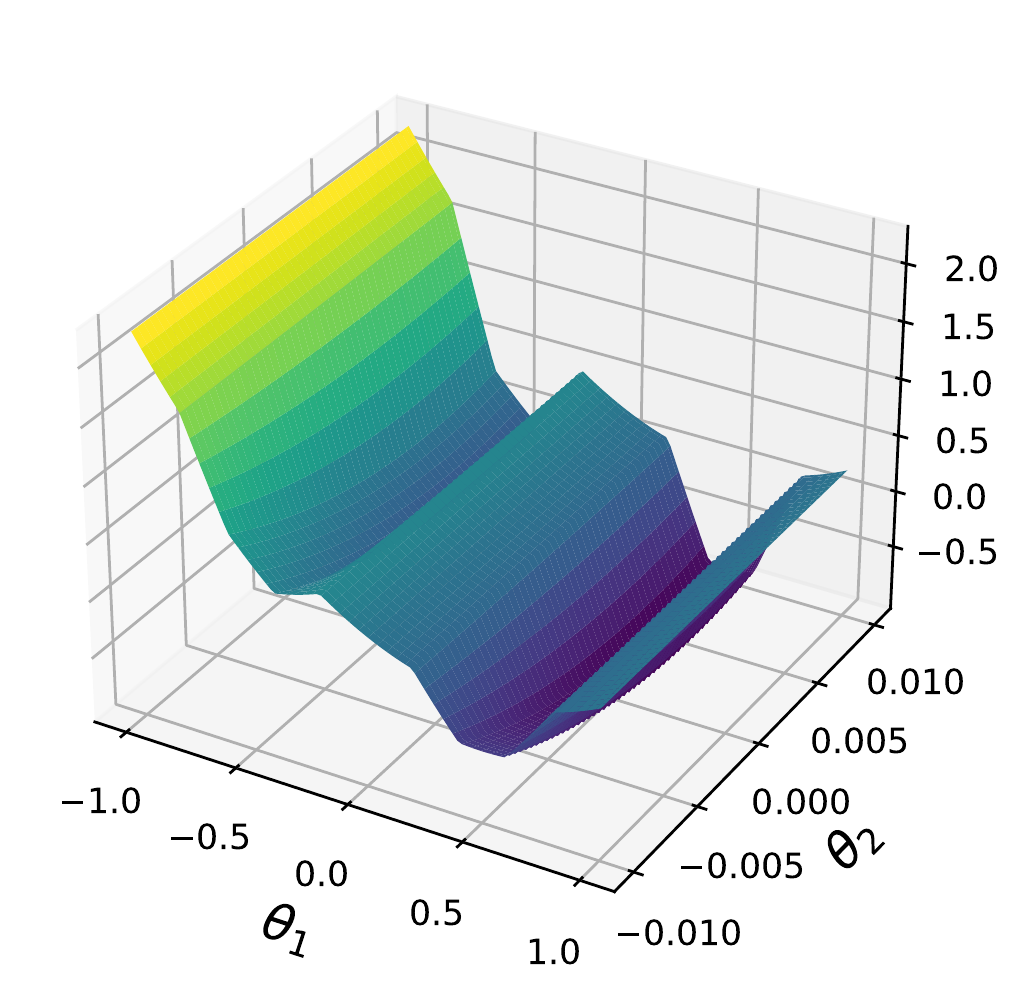}}
    \subfigure[]{\includegraphics[width=0.55\textwidth]{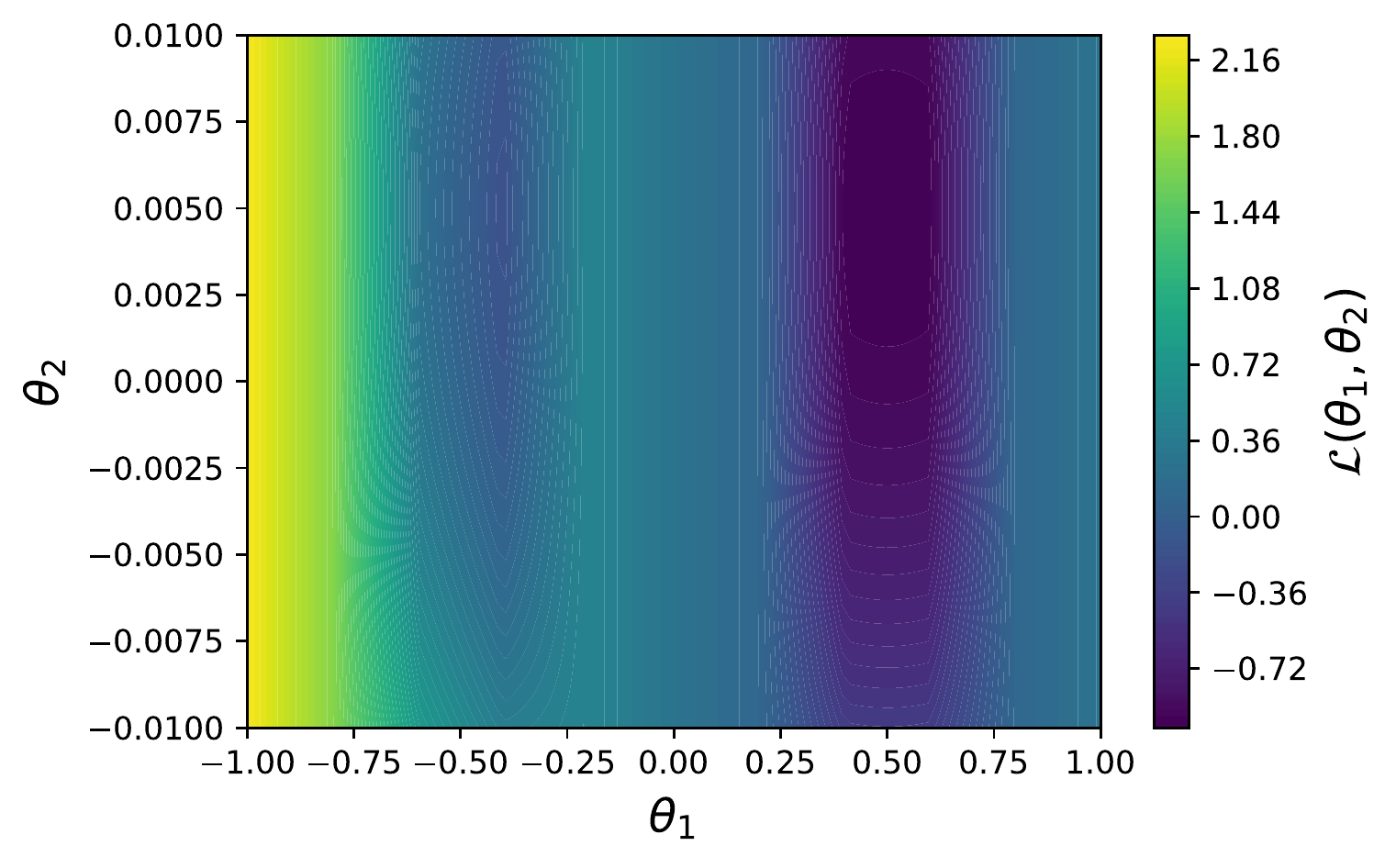}}
    \caption{(a) A 3-D visualization of the toy example of our construction for $\ell$, for one-dimensional $\theta_1$ and $\theta_2$. (b) A heatmap of the same example.}
    \label{fig:wells}
\end{figure}

%We  begin our lower bound construction by defining a set of data distributions $\calT$ and loss function $\ell$ that demonstrates that it is necessary to pretrain on public data and fine-tune on private data, and provide the intuition guiding our design.

\mypar{The loss functions} In the initial stage of the algorithm (outside of $S$), the lower bound for private algorithm  follows from the choice of $\ell_1(\theta_1; d_1) := \min\{(1/2) \ltwo{\theta_1 - d_1}^2, \frac{9}{2}\}$ defined over the first $p^4$ coordinates. Note that as long as $\ltwo{\theta_1} \leq 2$ and $d_1$ is in $\calD_1$, this is equivalent to  a loss function of $\ltwo{\theta_1 - d_1}^2$, i.e.~we can still apply Lemma~\ref{lem:privatesclb} to $\ell_1$. The minimum is used in our upper bound to keep $\ell_1$ bounded in the low-probability event that DP-SGD adds a large amount of noise to $\theta_1$.

For any $\calA_{priv}$, we define our basin to include the global minima of $\ell_1$ on the  distribution $\tau(\calA_{priv})$ in Lemma~\ref{lem:privatesclb}. Since we know $|\ltwo{\theta^*(\tau(\calA_{priv}))\|-M} =o_n(1)$, we let
\begin{eqnarray} 
S \;:=\; B_{p^4}(0, M + R_1) \setminus B_{p^4}(0, M - R_1)\;.\label{def:S}
\end{eqnarray}
where $M = \Omega(1)$ is defined as in \cref{lem:privatesclb} for the case when dimension is  $p^4$, $\epsilon = 1$, and $n_{priv} = p^2$, and we choose some $R_1 < M$. Note that $S$ is the set of all points where $\ell_2$-norm of $\theta_1$ is close to $M$; the basin is a single non-convex set. Our construction can seamlessly generalize to the case where there are numerous disconnected basins to resemble more realistic landscapes. If $R_1$ is sufficiently large, then Lemma~\ref{lem:privatesclb} guarantees that the population minimizer of $\ell_1$ is contained in $S$ and far from the boundary of $S$ for distributions in $\calT_1$ as defined in that lemma. Further, by a vector Azuma inequality \cite{hayes03vectorazuma} the same is true of the empirical minimizer of $\ell_1$ over the public data with high probability. We will specify a value of $R_1$ in Appendix~\ref{app:wells}.

In the next stage of the algorithm (inside $S$), the loss is dominated by $\ell_2(\theta_2, d_2) := \min\{0, \frac{\ltwo{\theta_2 - d_2}^2}{2r^2} -  \frac{9}{2}\}$ where we use $r$ to scale the domain of $\ell_2$. In particular, let $\calT_2'$ be the set of $p$-dimensional data distributions over $\calD_2' := B_p(0, r)$, and $\calT_2'$ is defined by shrinking the support of each distribution in $\calT_2$ (as defined in Lemma~\ref{lem:publicsclb}) by a factor of $r<1$. 
%Further, let $\calT_2'$ be the set of distributions over $\calD_2' := B_p(0, r)$ which is given by shrinking the sample space of each distribution in $\calT_2$ by a factor of $r<1$.  
We will specify the value of $r$ in Appendix~\ref{app:wells}; for now, one can think of $r \ll 1$.  Since rescaling does not fundamentally change the problem, again \cref{lem:publicsclb} (up to a $1/r^2$ rescaling) holds also in $\calT_2'$.

Note that as long as $\ltwo{\theta_2} \leq 2r$ and $d_2 \in \calD_2'$, minimizing $\ell_2$ is equivalent to minimizing $\frac{\ltwo{\theta_2 - d_2}^2}{2r^2}$, which is just a rescaling of minimizing $\frac{\ltwo{\theta_2 - d_2}^2}{2}$. In other words, we can still apply Lemma~\ref{lem:publicsclb} to $\ell_2$. Putting $\ell_1$ and $\ell_2$ together, our loss is defined in Eq.~\eqref{def:loss} with a choice of $R_2 < M - R_1$, which  implies $q(0) = 0$; the exact value of $R_2$ is immaterial to our construction and eventual theorem statement.

% \textcolor{blue}{
% }
% {\bf Sewoong: These can be moved to the appendix in an appropriate place.}

\begin{figure}[h!]
\begin{center}
\includegraphics[width=.5\textwidth]{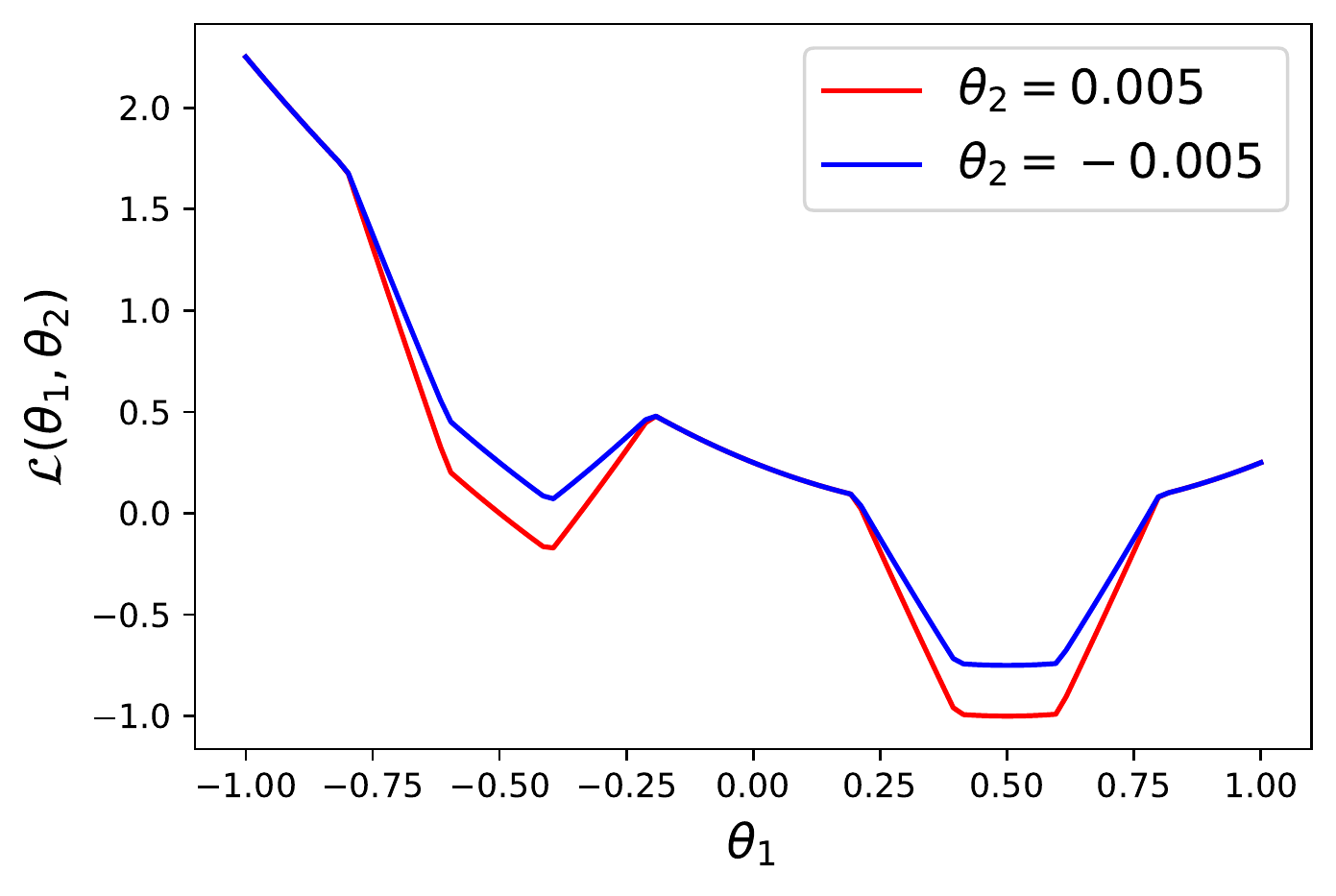}
\end{center}
\caption{A projection of our two-dimensional toy example loss onto $\theta_2 = 0.005$ and $\theta_2 = -0.005$.}
\label{fig:wells_projected}
\end{figure}

\mypar{Toy example of the loss function} In Figure~\ref{fig:wells} we provide a visualization of our loss $\ell( \cdot, d)$ for a single data point $d=(0.5, 0.005)$ as defined in Eq.~\eqref{def:loss} for $p=1$. Here, to simplify the visualization we have chosen $r = 0.01$, $M = 0.5$, $R_1 = 0.1$, $R_2 = 0.2$, which may not correspond to the actual values we choose in our construction. This gives $S = [-0.6, -0.4] \cup [0.4, 0.6]$, and $q(\theta_1) = 0$ if $\theta_1 \in [-\infty, -0.8] \cup [-0.2, 0.2] \cup [0.8, \infty]$. Since $0.5 \in S$ and thus $q(0.5) = 1$, the minimizer is $(0.5, 0.005)$.  We can observe the following. 

The first stage of the optimization (which corresponds to pretraining) tries to find the right part of the basin $S$ with small $\ell_1(\theta_1)$. For a fixed $\theta_2$, $\ell$ is a quadratic with respect to $\theta_1$, except for the  ``wells'' centered at $\theta_1 = 0.5$ and $ -0.5$ (Figure~\ref{fig:wells_projected}). In our construction, the population minimizer of $\theta_1$ would always be in one of the basins in $S = [-0.6, -0.4] \cup [0.4, 0.6]$. Note that the two basins are disconnected only because of the choice of $p=1$.  

The second stage of the optimization (which corresponds to fine-tuning) tries to  minimize the second loss $\ell_2(\theta_2)$. 
For a fixed $\theta_1$, $\ell$ is a quadratic with respect to $\theta_2$. The strong convexity of this quadratic increases with $q$; when $q(\theta_1) = 0$ (e.g. at $\theta_1 = 0)$ then $\ell$ is a constant with respect to $\theta_2$. 

In particular, we can see from Figure~\ref{fig:wells} that if we are at the origin, we can see that a (non-noisy) gradient step will only optimize over $\theta_1$, but once $\theta_1$ is inside $S$ then gradient steps will optimize both $\theta_1$ and $\theta_2$. Furthermore, if we start at $\theta_1$ in $S$ and run (DP) gradient descent, the scale of $\theta_2$, which is controlled by the choice of $r$ in the definition of $\ell_2$,  is much smaller than the scale of $\theta_1$. So it should be possible to optimize $\theta_2$ using gradient descent once $\theta_1$ is inside $S$, without causing $\theta_1$ to move very far. This roughly corresponds to fine-tuning staying within a basin in our hypothesis.

% ------------------------------------------------
\subsection{Analysis} 
\label{sec:id_analysis} 

With the above construction, we formally guarantee that for certain sizes of public and private datasets, both datasets are necessary to optimize the loss to a desired level. We defer the proof to Appendix~\ref{app:wells}. 

\begin{thm}\label{thm:quadratic-unconstrained}
For every integer $p \geq 1$, for some $r > 0$, $\calC = \mathbb{R}^{p^4} \times \mathbb{R}^p, \calD = B_{p^4}(0, 1) \times B_{p}(0, r)$ there exists $\ell$ and a set of distributions $\calT$ over $\calD$ such that:
\begin{enumerate}
    \item[(1)] For $\delta = o(1/p^2)$, any $(1, \delta)$-DP algorithm $\calA_{priv} : \calD^{p^2} \rightarrow \calC$, and any $\calA_{pub}: \calD^p \rightarrow \calC$ there exists $\tau \in \calT$ such that:
    \begin{align*}
    &\Expect{D \sim \tau^{p^2}, \theta \sim \calA_{priv}(D)}{\calL(\theta)} = \calL(\theta^*(\tau)) + \Omega(1),
    \end{align*}
    \begin{align*}
    &\Expect{D \sim \tau^p, \theta \sim \calA_{pub}(D)}{\calL(\theta)} = \calL(\theta^*(\tau)) + \Omega(1)
    \end{align*}
    \item[(2)]
    % For any (non-private) algorithm $\calA_{pub} : \calD^p \rightarrow \calC$ there exists $\tau(\calA_{pub}) \in \calT$ such that:
    % \begin{align*}
    % &\Expect{D \sim \tau(\calA_{pub})^n, \theta \sim \calA_{pub}(D)}{\calL(\theta)}\\
    % &= \calL(\theta^*(\tau(\calA_{pub}))) + \Omega(1)
    % \end{align*}
    % \item[(3)]
    For any $\delta \geq 2^{-p}$, there exists an algorithm $\calA_{mixed} : \calD^{p+p^2} \rightarrow \calC$ which runs gradient descent on the first $p$ examples, followed by $(1, \delta)$-DP-SGD on the last $p^2$ examples, such that for any $\tau \in \calT$:
    \[\Expect{D \sim \tau^n, \theta \sim \calA_{mixed}(D)}{\calL(\theta)} = \calL(\theta^*(\tau)) + \tilde{O}(1/p)\]
\end{enumerate}
\end{thm}
This demonstrates that there exist data distributions where a small number of public in-distribution data is necessary to achieve small loss, and pretraining on that public data is sufficient for DP-SGD to achieve the desired level of loss. The first part of the theorem shows that there are data distributions where neither a small-size, $n_{pub}=p$, public data or a large-size, $n_{priv}=p^2$, private data can reach the desired loss. However, on the same data distribution, pretraining on the small-size public data, followed by finetuning on the large-size private data, achieves a desired level, $O(1/p)$, of the excess loss. 

\begin{proof}[Proof Sketch of Theorem~\ref{thm:quadratic-unconstrained}]
The high-level idea behind the construction is: Using private data alone cannot achieve risk $o(1)$ on $\ell_1$, because $\ell_1$ has a high dimension, but using public data can achieve risk $O(1/p)$ because the public mean estimation risk guarantees are dimension-independent. Similarly, using public data alone cannot achieve risk $o(1)$ on $\ell_2$, because $\ell_2$ has a multiplier of $p$ and the amount of public data we are allowed to use is small. However, using private data can achieve risk $O(1/p)$ on $\ell_2$ because $\ell_2$ has low dimension, and there is more private data to use. 

To prove (1) using these observations, we show that the risk guarantee of $\calA$ on $\ell$ is at least its risk on $\ell_1$ or $\ell_2$ alone. If $\calA_{pub}$ only uses public data, this implies a lower bound on $\calA_{pub}$'s risk on $\ell$ from Lemma~\ref{lem:publicsclb}, which holds for some distribution $\tau_2 \in \calT'_2$. Similarly, if $\calA_{priv}$ only uses private data, this implies a lower bound on $\calA_{priv}$'s risk on  from Lemma~\ref{lem:privatesclb}, for some distribution $\tau_1 \in \calT_1$. Then, the product distribution $\tau = \tau_1 \times \tau_2$ gives a simultaneous lower bound on the risk of $\calA_{pub}$ and $\calA_{priv}$, as desired.

To prove (2), we observe that a single step of (full-batch) gradient descent on the public data takes $\theta_1$ to the empirical minimizer of $\ell_1$, which achieves risk $O(1/p)$ for $\ell_1$. If we use an initialization such that $q(\theta_1) = 0$, a single step of gradient descent has no effect on $\theta_2$, since the gradient of $\ell$ with respect to $\theta_2$ at the initialization is zero. Furthermore, if $R_1$ is sufficiently large, then with high probability after this single step $\theta_1 \in S$ and is far from the boundary of $S$, i.e. $q(\theta_1) = 1$ and we have $\ell = \ell_1 + p \cdot \ell_2$. Then, running DP-SGD with optimal parameters from this point will take $\theta_2$ to a point achieving risk $O(1/p)$ on $p \cdot \ell_2$. However, DP-SGD will also move $\theta_1$, which could worsen our risk on $\ell_1$ substantially. We show that if $r$ is sufficiently small, then for DP-SGD with optimal parameters, the amount by which $\theta_1$ moves is $O(1/p)$, and in turn $\theta_1$ remains in $S$ and the risk guarantee on $\ell_1$ does not worsen by more than $O(1/p)$. Then, our overall risk guarantee is at most the sum of the risk guarantee on $\ell_1$ and $p \cdot \ell_2$ individually, which is $O(1/p)$. 
\end{proof}

%We make a few remarks about Theorem~\ref{thm:quadratic-unconstrained} and its proof.

\subsection{Pretraining on out-of-distribution public data} 
\label{sec:ood} 

A more common setting in practice is when out-of-distribution large-scale public data is used in pretraining, as we surveyed in the introduction and at the beginning of \cref{sec:construct}. We modify our previous construction  in  Theorem~\ref{thm:quadratic-unconstrained} so that $(i)$ there is a distribution mismatch between the public and private examples and $(ii)$ an arbitrarily large amount, $n_{pub}$, of public data is available. 

\begin{thm}\label{thm:quadratic-mismatch}
For every integer $p \geq 1$ and $n_{pub} \geq p$, for some $r > 0$, $\calC = \mathbb{R}^{p^4} \times \mathbb{R}^p, \calD = B_{p^4}(0, 1) \times B_{p}(0, r)$ there exists $\ell$ and a set  $\calT$ of pairs of distributions $(\tau_{pub}, \tau_{priv})$ over $\calD$ such that:
\begin{enumerate}
    \item[(1)] For $\delta = o(1/p^2)$, any $(1, \delta)$-DP algorithm $\calA_{priv} : \calD^{p^2} \rightarrow \calC$, and any $\calA_{pub}: \calD^p \rightarrow \calC$ there exists $(\tau_{pub}, \tau_{priv}) \in \calT$ such that:
    \begin{align*}
    &\Expect{D \sim \tau_{priv}^{p^2}, \theta \sim \calA_{priv}(D)}{\calL(\theta)} = \calL(\theta^*(\tau)) + \Omega(1),
    \end{align*}
    \begin{align*}
    &\Expect{D \sim \tau_{pub}^{n_{pub}}, \theta \sim \calA_{pub}(D)}{\calL(\theta)} = \calL(\theta^*(\tau)) + \Omega(1)
    \end{align*}
    \item[(2)]
    % For any (non-private) algorithm $\calA_{pub} : \calD^p \rightarrow \calC$ there exists $\tau(\calA_{pub}) \in \calT$ such that:
    % \begin{align*}
    % &\Expect{D \sim \tau(\calA_{pub})^n, \theta \sim \calA_{pub}(D)}{\calL(\theta)}\\
    % &= \calL(\theta^*(\tau(\calA_{pub}))) + \Omega(1)
    % \end{align*}
    % \item[(3)]
    For any $\delta \geq 2^{-p}$, there exists an algorithm $\calA_{mixed} : \calD^{n_{pub}+p^2} \rightarrow \calC$ which runs gradient descent on the first $n_{pub}$ examples, followed by $(1, \delta)$-DP-SGD on the last $p^2$ examples, such that for any $\tau \in \calT$:
    \[\Expect{D \sim \tau_{pub}^{n_{pub}} \times \tau_{priv}^{p^2}, \theta \sim \calA_{mixed}(D)}{\calL(\theta)}= \calL(\theta^*(\tau)) + \tilde{O}(1/p)\]
\end{enumerate}

Here, $\calL$ refers to the population loss over $\tau_{priv}$.
\end{thm}

This demonstrates that there exist data distributions where out-of-distribution public data is necessary to achieve small test loss on the target private task, and pretraining on the OOD public data is sufficient for DP-SGD to achieve the desired test loss. Note that all three cases are evaluated on the same private population loss, as is the case in real-world scenarios where we care about the performance on the private task.  

We prove Theorem~\ref{thm:quadratic-mismatch} in Appendix~\ref{app:wells} and give here a proof sketch for what modifications of Theorem~\ref{thm:quadratic-unconstrained} are needed. In particular, the value of $d_2$ in the public examples is irrelevant in the upper bound in Theorem~\ref{thm:quadratic-unconstrained}. For example, we could have $d_2 = 0$ in all public examples, and the upper bound is unaffected. With the extra freedom we have in the construction under this distribution mismatch, showing the lower bound on the risk on the private task for algorithms using only public data is easy; the value of $d_2$ in the public examples encodes no information about the distribution of $d_2$ in the private examples, so clearly no algorithm with only access to public data can achieve good risk on $\ell_2$ alone, regardless of how much public data it has access to. 
%This holds for any size of public data.

{\bf Data abundance:} If we had $p^2$ ID public examples or $p^5$ private examples in Theorem~\ref{thm:quadratic-unconstrained}, we could achieve risk $O(1/p)$ in the above construction using only public data or only private data. Of course, if we also have the distribution mismatch in the preceding paragraph, no amount of public data achieves low risk on the private population.
In light of this, Theorem~\ref{thm:quadratic-unconstrained} should not be interpreted as saying that both public and private data are strictly necessary to optimize some loss functions. Instead, a better interpretation might be that a small amount of public data greatly reduces the amount of private data needed to solve an optimization problem. This can be seen as theoretical backing for an empirical observation made in \cite{tramer2020differentially,de2022unlocking,li2021large,kerrigan2020differentially}. 

{\bf Convex losses:} Our construction is inherently non-convex, due to the term $q(\theta_1)$ we use to ``activate'' $\ell_2$ only after optimizing over the public data. Surprisingly, in Appendix~\ref{app:quadratic} we show that Theorem~\ref{thm:quadratic-unconstrained} can be proven even for (non-isotropic) quadratic losses, at the cost of operating in a constrained setting (i.e. $\calC$ is finite). The constrained requirement is necessary since unlike in the construction in this section, 
%a quadratic's gradient with respect to any coordinate is not zero at most points. So unlike for $\ell$ constructed in this section,
we cannot guarantee that gradient descent on the public data does not affect $\theta_2$. However, in the constrained setting we have the guarantee that $\theta_2$ cannot leave the constraint set, 
%and from the perspective of DP-SGD with optimal parameters, all initializations within the constraint set are reasonable,
so it is okay to take (arbitrarily large) public gradient steps that affect $\theta_2$.
\section{Experiments}
\label{sec:exp} 

In this section, we conduct experiments to verify our hypothesis about the two-stage optimization phenomenon. 

\subsection{CIFAR10 Experiments}\label{sec:cifar}

{\bf Setup:} For the ID public data experiment in Figure~\ref{fig:pre_post_public} (left), we train a ConvNet model on CIFAR10 using DP-SGD. We train for 60 epochs with a clipping norm of one, learning rate of 0.001, batch size of 256, and Adam optimizer. Simulating an ID public data setting, we split CIFAR10 (60,000 images) into  a public dataset of size 2,000 and a private dataset of size 58,000. 
We use Adam optimizer with learning rate of 0.002 for the public dataset. For the large-size OOD public data in Figure~\ref{fig:pre_post_public} (right), 
we used 20,000 images from the training part of the CINIC10 images as the public data.

{\bf Results:} In Figure~\ref{fig:pre_post_public},  we allow  a limited number of epochs $T_{pub}$ on the public data. We show test accuracy as a function of $t$ (the x-axis), which is the number of epochs used in public pretraining. The remaining $T_{pub}-t$ epochs are used in public post-training after the private training. For ID public data in the left panel, we choose $T_{pub} = 200$. Using this budget for pretraining has the highest accuracy. This demonstrates that the initial rounds of training are the most sensitive to noise, as is the case in both our hypothesis from Section~\ref{sec:intro} and our theoretical construction in Section~\ref{sec:construct}. Note that the benefits of longer pretraining is small after $t = 100$. It is possible that after around 100 epochs, pretraining converges to a good basin and  the benefits of public pretraining plateaus afterwards. We see the same trend with OOD public data using CINIC10 dataset with $T_{pub} = 30$, shown in Figure~\ref{fig:pre_post_public} (right). Again, we observe that reducing privacy noise in the earlier rounds of training is more beneficial.

\begin{figure}[t]
    \centering
    \includegraphics[scale=0.49]{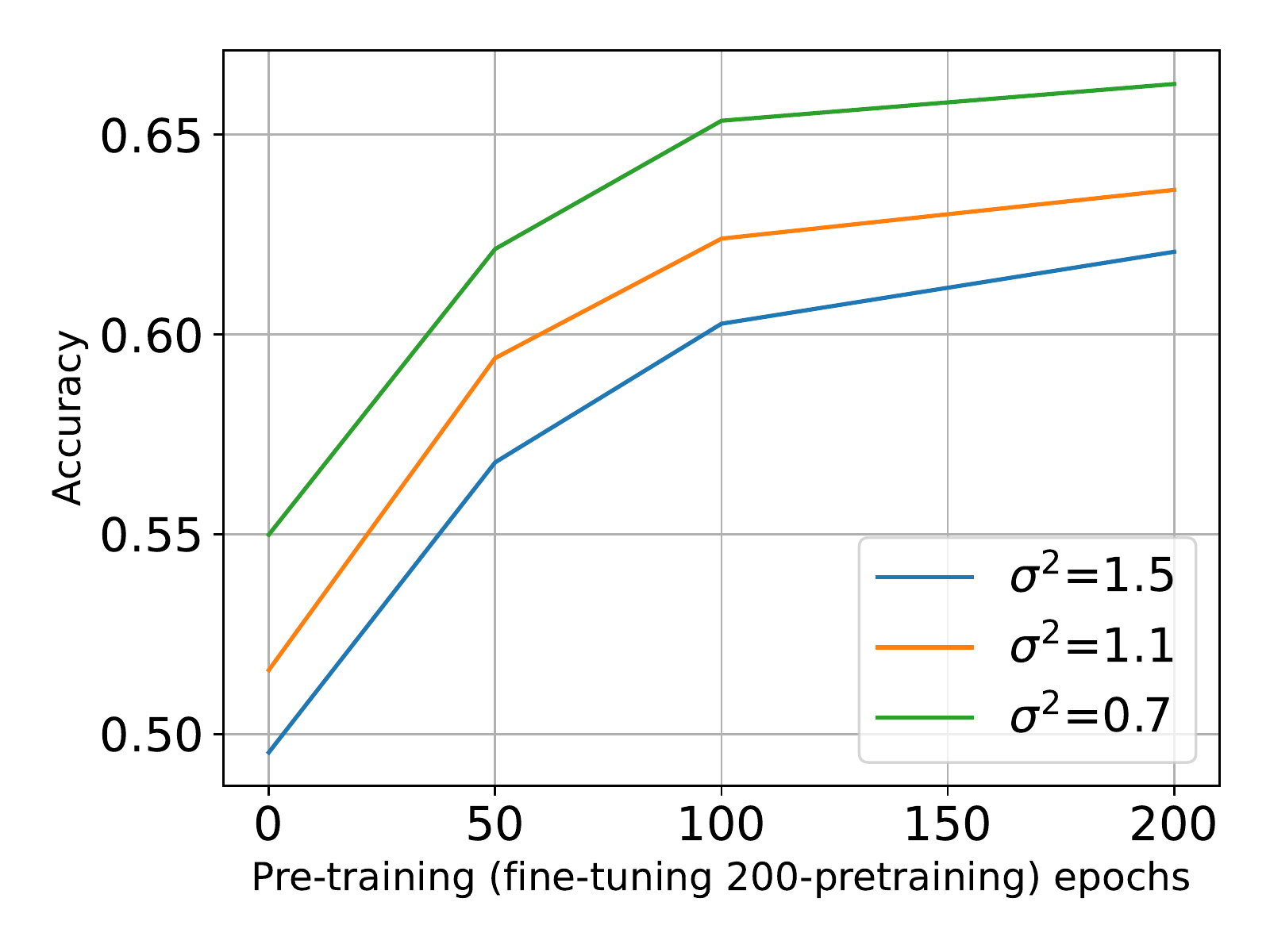} 
    \includegraphics[scale=0.49]{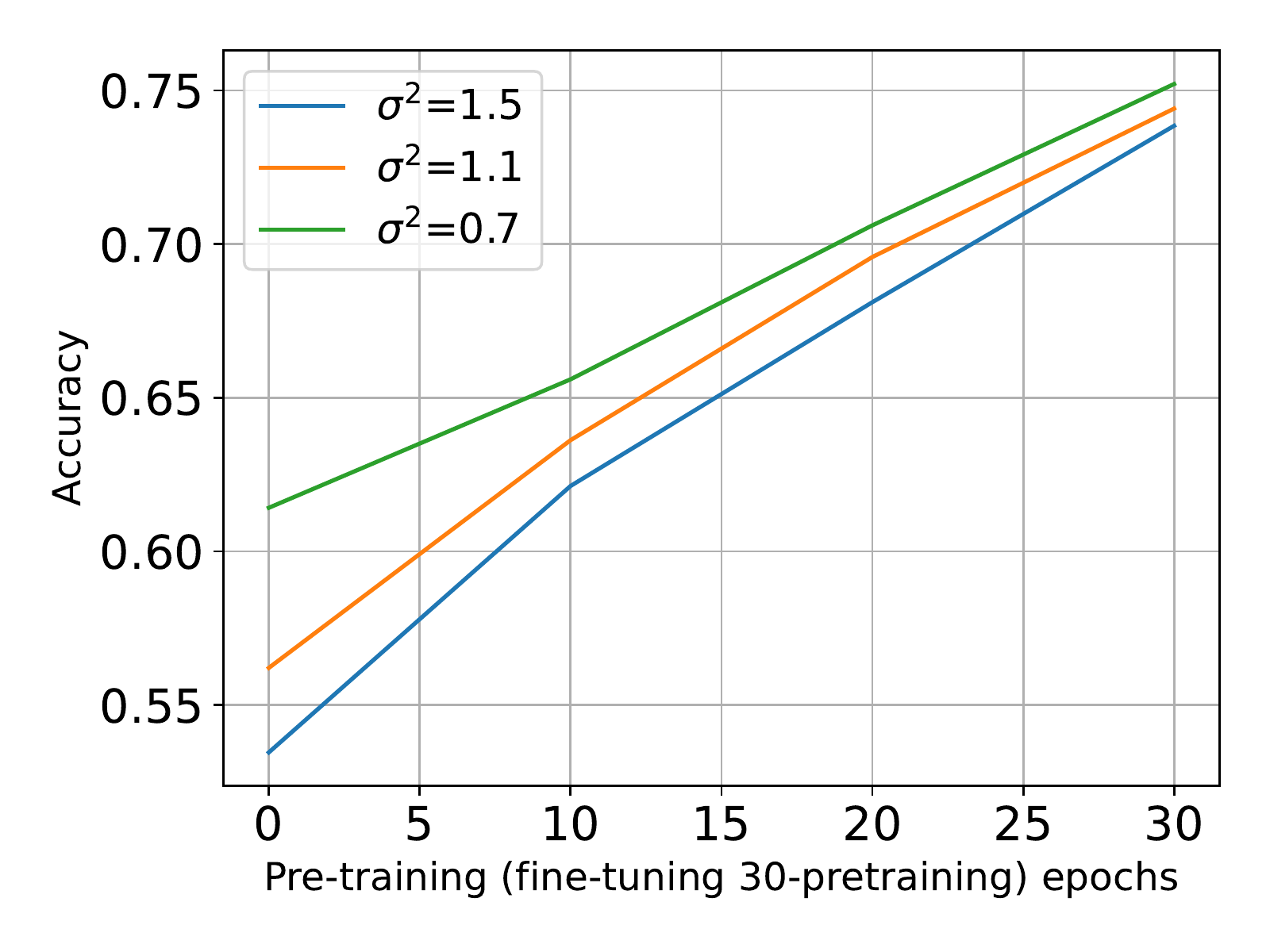}
    \caption{On CIFAR10, pretraining on the public data  significantly improves accuracy compared to  post-training on the public data for both ID public data (left) and OOD public data (right).}
    \label{fig:pre_post_public}
\end{figure}

To further demonstrate the importance of the earlier iterations in the training, we designed an experiment where instead of using the same privacy budget for all iterations of private training, we train the first iteration with a lower noise multiplier (using more privacy budget) and compared it a setting where we train the last iteration with the lower noise multiplier. Table~\ref{tab:dp_high_budget} compares the results for various choices of the end-to-end $\varepsilon$.
Again, we observe that reducing privacy noise in the earlier rounds of training is more beneficial.
  
\begin{table}[]
    \centering
    \begin{tabular}{ccc}
    \toprule
        $\varepsilon$ & first epoch &  last epoch \\
    \midrule
        1 & 46.7$\%\pm$ 0.3   & 46.3$\% \pm$ 0.3  \\
        3 & 49.6$\%\pm$ 0.6  & 48.0$\% \pm$ 0.5 \\
        8 & 54.0$\%\pm$ 0.8  & 52.0$\% \pm$ 0.9 \\
    \bottomrule
    \end{tabular}
    \caption{Effect of having higher budget ($\sigma^2=0.6$) on the first epoch compared to the last epoch on CIFAR10.}
    \label{tab:dp_high_budget}
\end{table}

\subsection{Manifold on Large Speech Model}\label{sec:manifold}

\paragraph{Setup}
To better understand the geometry of the loss function for training machine learning models, we evaluate training a ConformerM~\cite{gulati2020conformer} model on Librispeech~\cite{panayotov2015librispeech} dataset with/without public data pretraining using DP-Adam.
Specifically, we train the following three models:
\begin{itemize}
    \item\textbf{Oracle model:}  We train a ConformerM model on the complete Librispeech dataset for 100k steps. This is considered as the global minima of the manifold.
    \item\textbf{Private model:} We train a ConformerM model on 90\% samples drawn uniformly from the Librispeech dataset using DP-Adam for 20k steps.
    \item\textbf{Private model with public pretraining:} We pretrain a ConformerM model on the 10\% of the samples with Adam for 10k steps and then fine-tune on the remaining 90\% samples with privacy for 1k steps. 
\end{itemize}
Note that the hyper-parameters for the latter two settings are tuned to optimize the test word error rate under the same privacy budget $\epsilon=9.8$.
We fix the privacy parameter $\delta$ to $10^{-6}$, ensuring that $\delta < n^{-1}$, where $n$ is the number of private samples.

\paragraph{Results}
As shown in Figure~\ref{fig:manifold}, we interpolate the three models above to draw a projected slice of the manifold.
From both the heatmap and the contour figures, we can tell that private model with public training falls into the same ``basin'' as the oracle model, which we refer to as the \textit{global minima basin}. The private model without pretraining falls into a different basin, separated from the global minima basin by a ``hill''.  This is evidence for our hypothesis that public pretraining is useful specifically because it picks a good basin. The $\ell_2$-distance between the oracle model and the private model with public pretraining is 671.22, much smaller than the distance between the oracle model and the private model which is 1738.27. This parallels the construction in Section~\ref{sec:construct}, in which private fine-tuning takes place on a smaller scale than pretraining on public data.

\begin{figure}[h]
    \centering
    \subfigure[]{\includegraphics[width=0.45\textwidth]{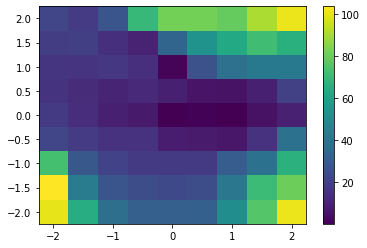}}
    \subfigure[]{\includegraphics[width=0.45\textwidth]{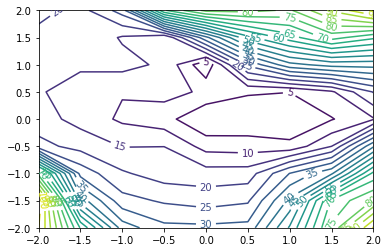}}
    \caption{Projected manifold of ConformerM on Librispeech by interpolating 3 models. (0, 0) is the oracle model. (0, 1) is the private model. (1, 0) is the private model with public pretraining. We can tell that (0, 0) and (1, 0) are within the same basin while (0, 1) is in a different basin separated by a hill on the manifold. The manifold is constructed by calculating cross entropy loss on a 128-sample subset of Librispeech's testother dataset.}
    \label{fig:manifold}
\end{figure}

\section{Discussion}
\label{sec:discuss} 

In this paper, we show  that there exist natural learning tasks where public data is necessary and sufficient 
to achieve a target accuracy under DP model training. This conclusion is independent of whether the public data is in-distribution with the private training data or not. Recently, \cite{tramer2022considerations} discussed the perils of indiscriminate use of public data in DP training, few of which are: $(i)$ Publicly available data does not necessary mean that one can use that dataset for training models without privacy consideration, as the trained model can release information from the dataset verbatim, and $(ii)$ Many existing empirical works on achieving better accuracy for DP training by using public data do not necessarily reflect realistic scenarios for model training. In particular, in real-world settings the available public data can be far out-of-distribution from the private dataset. The authors provide prescriptive recommendations on being judicious in the choice of public data for DP training. Our work is complementary to ~\cite{tramer2022considerations}, and we concur with all the concerns in their paper. Given our current impossibility result, and the concerns in~\cite{tramer2022considerations}, an important research question for future exploration is~{\it given a public dataset which may be far out-of distribution from the private training data, what is the best DP training procedure that exploits the public dataset to obtain higher accuracy?} To the best of our knowledge, all current works (see Section~\ref{sec:intro} for reference) on the use of public data in DP training do not provide an answer. 

Another question raised by our work is whether one can use the insight that the geometry changes after public pretraining to improve private fine-tuning in practice. That is, in practice the ideal algorithm for training from scratch on private data may not be the ideal algorithm for fine-tuning a pretrained model. While some works \cite{song21clipping, li2022does} observe that DP-SGD can inherently benefit from the geometry of the loss having certain properties, and others \cite{asi2021private, amid2022mirrordescent} develop algorithms that adapt to the local geometry, we are not aware of any work that uses properties of the geometry specific to the fine-tuning phase. One possibility is that if the loss function is locally convex after public pretraining, theoretical techniques whose utility guarantee depends on convexity might offer larger improvements for private fine-tuning in practice than for training from scratch. 

In our experiments, we observed that the first phase of model training is more sensitive to noise when we do fully private training from random initialization. 
As our two-phase hypothesis suggests, this phenomenon only occurs in non-convex optimization. In convex optimization, an opposite strategy of reducing the privacy noise towards the end of training helps more, as theoretically analyzed and empirically demonstrated in \cite{lee2018concentrated}. For non-convex optimization, \cite{hong2022dynamic} proposes the use of a decreasing noise multiplier under a strong condition on the loss function known as the Polyak-Lojasiewicz condition. This implies that vanilla gradient descent converges fast and does not apply for the typical loss landscapes of deep learning problems. For typical non-convex optimization experiments in our setting, smaller privacy noise at the beginning of the training improves performance (compared to having smaller privacy noise at the end of training). Of course, the need to choose a strategy for scheduling the privacy budget adds a hyperparameter for the training process. In particular, this hyperparameter is not needed if we do public pretraining instead. It would be useful for practitioners to give guidelines on how to schedule the privacy budget across training rounds such that we improve over a fixed noise multiplier, while minimally increasing the number additional hyperparameters to tune.

\bibliographystyle{alpha}
\bibliography{ref}
\newpage

\appendix 
\section{Notation Reference}
\label{app:notation}
\begin{table}[ht]
\begin{center}
\begin{tabular}{|c||c|c||c|c|}
\hline
 \textbf{Notation} & \textbf{Meaning} \\
\hline
$\calA$ & algorithm, i.e. a randomized map from datasets to outputs \\ \hline
$B_p(c, r)$ & $p$-dimensional ball of radius $r$ centered at $c$ \\ \hline
$\calC$ & constraint set \\ \hline
$D$ & data set ($\in \calD^*$)\\ \hline
$d$ & (singular) data point ($\in \calD$)\\ \hline
$\calD$ & data domain \\ \hline
$\epsilon, \delta$ & privacy parameters \\ \hline
$\eta$ & step size in gradient descent \\ \hline
$\ell$ & (per-example) loss function \\ \hline
$\calL$ & population loss function \\ \hline
$n$ & dataset size \\ \hline
$N(\mu, \Sigma)$ & normal distribution with mean $\mu$ and covariance matrix $\Sigma$ \\ \hline
$p$ & dimension \\ \hline
$\Pi$ & projection operator \\ \hline
$q$ & the ``activation function'' in Section~\ref{sec:construct} \\ \hline
$R, r$ & radii of various sets in our construction \\ \hline
$\tau$ & data (population) distribution \\ \hline
$\calT$ & set of data distributions \\ \hline
$\theta$ & model \\ \hline
$\theta^*(\tau)$ & population minimizer on the distribution $\tau$ \\ \hline
\end{tabular}
\end{center}
\caption{Summary of notation}
\label{tb:notation}
\end{table}

In Table~\ref{tb:notation}, we give a summary of the notation used throughout the paper.

\section{Survey of the gain of pretraining} 
\label{app:survey}

The stark difference in the gain of public pretraining between private and non-private model training has been widely observed in several tasks both in natural language and vision.

\medskip 
\noindent
{\bf Table-To-Text Generation.}
The experiment details can be found in \citep{li2021large}. 

\begin{table}[h]
  \begin{center}
        \begin{tabular}{|c| r r r| r r r |} 
        \hline
   & \multicolumn{3}{|c|}{BLEU}   & \multicolumn{3}{|c|}{ROUGE-L}\\ \hline 
   $\eps$ & $\infty$ & $8$ & 
   $3$ & $\infty$ & $8$ & 
   $3$ \\ 
 \hline 
 with public pretrain & 69.46 & 63.19 & 61.52 & 71.36 & 66.43 & 65.67 \\
 without public pretrain & 65.73 & 24.25& 15.46 & 68.75 & 39.95 & 35.24 \\ 
 \hline 
  gain of public pretraining & 3.73 & 38.94 & 46.06 & 2.61 & 26.48 & 30.43 \\
 \hline
    \end{tabular}
    \caption{BLEU and ROUGE-L scores for generating natural language descriptions of table entries on E2E dataset \citep{novikova2017e2e} reported in \citep[Table 2]{li2021large} with $\delta=10^{-5}$. The pretrained model in the first row is GPT-2.  }
    \label{tab:e2e}
    \end{center}
\end{table}

\begin{table}[h]
  \begin{center}
        \begin{tabular}{|c| r r r| r r r |} 
        \hline
   & \multicolumn{3}{|c|}{BLEU}   & \multicolumn{3}{|c|}{ROUGE-L}\\ \hline 
   $\eps$ & $\infty$ & $8$ & 
   $3$ & $\infty$ & $8$ & 
   $3$ \\ 
 \hline 
 with public pretrain & 42.78 & 35.06 & 31.03 & 56.72 & 54.58 & 52.06 \\
 without public pretrain & 26.79 & 7.77 & 3.00 & 37.86  & 21.68 & 17.14  \\ 
 \hline 
  gain of public pretraining & 15.99  & 27.29  &28.03  & 18.86   &  32.90  &  34.92
  \\
 \hline
    \end{tabular}
    \caption{BLEU and ROUGE-L scores for generating natural language descriptions of table entries on DART dataset \citep{nan2020dart} reported in \citep[Table 8]{li2021large} with $\delta=10^{-5}$. The pretrained model in the first row is GPT-2.  }
    \label{tab:dart}
    \end{center}
\end{table}

\medskip\noindent 
{\bf Image classification.} The experimental details can be found in \citep{de2022unlocking}.

\begin{table}[h]
  \begin{center}
        \begin{tabular}{|c| r r r r| r r r r|} 
        \hline
   & \multicolumn{4}{|c|}{CIFAR-10}   & \multicolumn{4}{|c|}{ImageNet}\\ \hline 
   $\eps$ & $8$ & $4$ & 
   $2$ & $1$ & $8$ & 
   $4$ & $2$ & $1$ \\ 
 \hline 
 with public pretrain &  96.7&96.1&95.4&94.7&81.8&79.2&74.7&70.3\\
 without public pretrain & 81.4 & 73.5& 65.9 & 56.8 & 32.4 & $-$&$-$&$-$  \\ 
 \hline 
  gain of public pretraining & 15.3 & 22.6 & 29.5  & 37.9 & 49.4 & $-$&$-$&$-$
  \\
 \hline
    \end{tabular}
    \caption{Test accuracy for image classification on CIFAR-10 and ImageNet datasets reported in \citep[Table 1]{de2022unlocking} with $\delta=10^{-5}$ and $8\cdot 10^{-7}$, respectively. The pretraining public data for CIFAR-10 is ImageNet and for ImageNet is JFT-4B. Without pretraining, private training on ImageNet failed to converge, indicated by $-$.}
    \label{tab:cv}
    \end{center}
\end{table}

% -------------------------------------
\section{Missing Details from Section~\ref{sec:construct}} 

Before turning to the proof, we fill in the details of the construction given in Section~\ref{sec:construct}:
We choose $R_1 = 1/p^2 + \kappa \log(p)/\sqrt{p}$, where $\kappa$ is a sufficiently large constant. Any $R_2 < M - R_1$ suffices for our proof. We choose $r = O(\frac{1}{p^{5/2} \sqrt{\log(1/\delta)}})$. We formally define the range of data distributions we use in our construction as a set of products of two distributions: $\calT := \{\tau_1 \times \tau_2 | \tau_1 \in \calT_1, \tau_2 \in \calT_2'\}$, where $\calT_1$ is defined as in Lemma~\ref{lem:privatesclb}, and $\calT_2$ is defined as in Section~\ref{sec:construct}.

\subsection{Proof of Theorem~\ref{thm:quadratic-unconstrained}}\label{app:wells}

In our proof we will use DP-SGD as instantiated in \cite{BST14}. Combined with results on uniform stability of gradient descent on strongly convex losses (see e.g. \cite{HardtRS16}), Theorem 2.4 of \cite{BST14} and its proof implies the following:

\begin{theorem}\label{thm:sc-dp-sgd}
Suppose $\ell$ has Hessian $m\mathbb{I}_p$ and for any $d, d'$, $\ltwo{\nabla \ell(\theta; d) - \nabla \ell(\theta; d')} \leq L$. Then for $T = n^2$, $\eta_t = \frac{1}{mt}$, $\sigma^2 = O(\frac{L^2 \log(1/\delta)}{\epsilon^2 n^2})$, if $\ltwo{\theta_0 - \theta^*} \leq O(L / m)$ running $T$ steps of DP-SGD with step size $\eta_t$ in iteration $t$ and variance $\sigma^2$ is $(\epsilon, \delta)$-DP and achieves population loss: 

\[O\left(\frac{L^2 p \log(n) \log (1/\delta)}{m \epsilon^2 n^2} + \frac{L^2}{m n}\right).\]
\end{theorem}

We also have the following lemma, which effectively says that unconstrained and DP-SGD stays within a ball with high probability.

\begin{lem}\label{lem:dpsgd-stays-in-ball}
With probability $1 - T 2^{-\Omega(p)}$ over (unconstrained) DP-SGD using the parameters in Theorem~\ref{thm:sc-dp-sgd}, for all $0 \leq t \leq T$, we have $\ltwo{\theta_t - \theta^*} \leq \max\{2\sqrt{p}\sigma/m, \ltwo{\theta_0 - \theta^*}\}$.
\end{lem}
\begin{proof}
By a multivariate Gaussian tail bound, w.p. $1 - T 2^{-\Omega(p)}$ in each iteration of DP-SGD the noise we add has $\ell_2$-norm at most $2\sqrt{p} \eta_t \sigma$. Conditioned on this event, since we have an identity quadratic loss and $\eta_t \leq 1/m$ for all $t$, each step of gradient descent is $(1-m \eta_t)$ contractive. So we have:

\[\forall t: \ltwo{\theta_t - \theta^*} \leq (1 - m \eta_t) \ltwo{\theta_{t-1} - \theta^*} + 2 \sqrt{p}\eta_t \sigma.\]

The lemma follows by induction.
\end{proof}

\begin{proof}[Proof of Theorem~\ref{thm:quadratic-unconstrained}]
We use $S$, $\ell$ as defined in Section~\ref{sec:construct}, (and the associated definitions of $\calD, \calT, q$, etc.).

We will prove (1) in two parts. First, we will show that for any $\calA_{priv}$, there exists $\tau_1 \in \calT_1$ such that the desired lower bound holds for $\tau_1 \times \tau_2$ for all $\tau_2 \in \calT_2'$. Second, we show that for any $\calA_{pub}$  there exists $\tau_2 \in \calT_2'$ such that the desired lower bound holds for $\tau_1 \times \tau_2$ for all $\tau_1 \in \calT_1$. Then taking $\tau_1$ from the first statement and $\tau_2$ from the second statement, both lower bounds hold for $\tau_1 \times \tau_2$ as desired.

\mypar{Proof of (1) for $\calA_{priv}$} 
Let $\tau_2$ be an arbitrary, fixed member of $\calT_2'$. Fix any $\calA_{priv}$ and take any distribution $\tau_1 \in \calT_1$. Let $\tau(\tau_1) \in \calT$ be the distribution over $(d_1, d_2)$ given by sampling $d_1 \sim \tau_1$ and $d_2 \sim \tau_2$. Let $\calL$ refer to the population loss over $\ell$, and let $\calL_1$ refer to the population loss over $\ell_1(\theta_1, d_1)$ for $d_1 \sim \tau_1$.  Consider the following algorithm $\calA_{priv}'$ for minimizing $\ell_1$ given $p^2$ samples from $\tau_1$: For each of these samples $d_1$, $\calA_{priv}'$ draws an i.i.d. sample $d_2$ from $\tau_2$ and pads $d_1$ with $d_2$, giving $p^2$ i.i.d samples from $\tau(\tau_1)$. $\calA_{priv}'$ then runs $\calA_{priv}$ on these samples, and takes $\theta_1$ from the output of $\calA_{priv}$. Note that $\calA_{priv}'$ is allowed to know the distribution $\tau_2$ since it is fixed and independent of the data $\calA_{priv}$ receives. We observe a few facts about $\calL$. First:

\begin{equation}\label{eq:theta2equivalence}
\calL((\theta_1, \theta_2)) - \calL((\theta_1, \theta_2')) = q(\theta_1)(\calL_2(\theta_2) - \calL_2(\theta'_2)).
\end{equation}

Since $q$ is non-negative, this gives:
\begin{equation}\label{eq:theta2monotonicity}
    \calL_2(\theta_2) \leq \calL_2(\theta_2') \leftrightarrow \calL((\theta_1, \theta_2)) \leq \calL((\theta_1, \theta_2'))
\end{equation}
This implies that replacing $\theta_2$ with the population minimizer of $\calL_2$ can only improve our risk on $\ell$. Next, note that for any $\tau_1 \in \calT_1$, since its population minimizer is in $S$, $q(\theta^*(\tau_1)) = 1$. In turn, $\theta^*(\tau_1 \times \tau_2) = (\theta^*(\tau_1), \theta^*(\tau_2))$ for all $\tau_1 \in \calT_1$. Finally, since $q$ is in $[0, 1]$ and $\calL_2$ is non-positive this gives:

\begin{equation}\label{eq:qatopt}
    \calL((\theta_1, \theta^*(\tau_2))) - \calL(\theta^*(\tau(\tau_1))) = \calL_1(\theta_1) - \calL_1(\theta^*(\tau_1)) - (1 - q(\theta_1)) \cdot \calL_2(\theta^*(\tau_2)) \geq \calL_1(\theta_1) - \calL_1(\theta^*(\tau_1)).
\end{equation}

In other words, if we choose $\theta_2$ to be the minimizer of $\calL_2$, then our risk on $\calL$ is at least our risk on $\calL_1$ alone. Putting it all together:

\[\Expect{D \sim \tau(\tau_1)^{p^2}}{\Expect{\theta \sim \calA_{priv}(D)}{\calL(\theta)} - \calL(\theta^*(\tau(\tau_1)))}\]
\[\stackrel{\eqref{eq:theta2monotonicity}}{\geq} \Expect{D \sim \tau(\tau_1)^{p^2}}{\Expect{(\theta_1, \theta_2) \sim \calA_{priv}(D)}{\calL((\theta_1, \theta^*(\tau_2)))} - \calL(\theta^*(\tau(\tau_1)))}\]
\[\stackrel{\eqref{eq:qatopt}}{\geq} \Expect{D \sim \tau_1^{p^2}}{\Expect{\theta_1 \sim \calA_{priv}'(D)}{\calL_1(\theta_1)} - \calL_1(\theta^*(\tau_1))}.\]

Using Lemma~\ref{lem:privatesclb}, since we are solving a $p^4$-dimensional mean estimation problem with $p^2$ samples and $(1, o(1/n)$-DP, we know that the final expression (the risk of $\calA_{priv}'$) is $\Omega(1)$ for some $\tau_1 \in \calT_1$, which implies the same lower bound on the risk of $\calA_{priv}$ for $\tau_1 \times \tau_2$.

\mypar{Proof of (1) for $\calA_{pub}$} Fix an arbitrary $\tau_1 \in \calT_1$. Let $\tau(\tau_2)$ denote $\tau_1 \times \tau_2$. Given any $\calA_{pub}$, consider $\calA_{pub}'$ that takes $p$ samples from $\tau_2$, pads them with i.i.d. samples from $\tau_1$ to get $p$ samples from $\tau(\tau_2)$. It then runs $\calA_{pub}$ on these samples, clips the norm of the $\theta_2$ in $\calA_{pub}$'s output to be at most $r$, and uses this as its output. 

We again make some observations on $\ell$. First, since the population minimizer $\theta^*(\tau_1)$ is always in $S$ by definition, we have $q(\theta^*(\tau_1)) = 1$. Then, since $\ell_2$ is non-positive:
\begin{equation}\label{eq:l1-population-minimizer}
 \calL(\theta_1, \theta_2) \geq  \calL(\theta^*(\tau_1), \theta_2)
\end{equation}

Next, by definition of $\ell_2$ and non-expansiveness of Euclidean projection, clipping $\theta_2$ to a ball of radius $r$ can only decrease the loss, i.e., if we define $\textsc{clip}(\theta_2, r) := \frac{\theta_2}{\max\{1, \ltwo{\theta_2}/r\}}$:
\begin{equation}\label{eq:l2-clipping}
 \calL_2\left(\textsc{clip}(\theta_2, r)\right) \leq \calL_2(\theta_2).
\end{equation}

Then we have:

\[\Expect{D \sim \tau(\tau_2)^p}{\Expect{\theta \sim \calA_{pub}(D)}{\calL(\theta)} - \calL(\theta^*(\tau(\tau_2)))}\]
\[\stackrel{\eqref{eq:l1-population-minimizer}}{\geq} \Expect{D \sim \tau(\tau_2)^p}{\Expect{(\theta_1, \theta_2) \sim \calA_{pub}(D)}{\calL((\theta^*(\tau_1), \theta_2))} - \calL(\theta^*(\tau(\tau_2)))}\]
\[\stackrel{\eqref{eq:theta2equivalence}}{=}\Expect{D \sim \tau(\tau_2)^p}{\Expect{(\theta_1, \theta_2) \sim \calA_{pub}(D)}{\calL_2(\theta_2)} - \calL_2(\theta^*(\tau_2))}\]
\[\stackrel{\eqref{eq:l2-clipping}}{\geq} \Expect{D \sim \tau_2^p}{\Expect{\theta_2 \sim \calA_{pub}'(D)}{\calL_2(\theta_2)} - \calL_2(\theta^*(\tau_2))}.\]

Since $\theta_2$ in the last line has norm at most $r$, the last expression (the risk of $\calA_{pub}'$ on $\ell_2$) is equal to $p$ times the risk of $\calA_{pub}'$ on a rescaling of the mean-estimation problem in Lemma~\ref{lem:publicsclb}, which is $\Omega(1)$ for some $\tau_2 \in \calT_2'$. This implies the same lower bound on the risk of $\calA_{pub}$ for $\tau_1 \times \tau_2$.

\mypar{Proof of (2)} 
We initialize $\theta_1 = \theta_2 = 0$ for simplicity\footnote{As long as $q(\theta_1) = 0$ and $\ltwo{\theta_2} \leq r$ initially with high probability, the proof still goes through.}. Then, note that the term $p \cdot q(\theta_1) \cdot \ell_2(\theta_2)$ is a constant in the region where $q(\theta_1) = 0$, which includes the origin. So, the gradient of this term is 0 at the origin, and a single step of gradient descent on the public data sets $\theta_1$ to the empirical minimizer of $\ell_1$ (which achieves risk $O(1/p)$ on $\ell_1$ alone in expectation), and does not affect $\theta_2$. We will then run DP-SGD from this point.

By a vector Azuma inequality, with probability at least $1 - p^{-\Omega(\kappa)}$, $\theta_1 \in S$ and is distance at least $\Omega(\frac{\log p}{\sqrt{p}})$ from the boundary of $S$. 
In the $p^{-\Omega(\kappa)}$ probability event this does not happen, since both $\ell_1$ and $\ell_2$ take on values in an interval of length $O(p)$, our risk is at most $O(p)$, and so for sufficiently large constant $\kappa$ the contribution of this event to our expected risk is negligible. So we just need to show our overall risk is $O(1/p)$ in expectation when after a single step of gradient descent on the public data, $\theta_1 \in S$ and is distance $\Omega(\frac{\log p}{\sqrt{p}})$ from the boundary. 

Note that as long as $\theta_1 \in S$, $q(\theta_1) = 1$ and thus the Lipschitzness of $\ell$ with respect to $\theta_1$ is $O(1)$. We will argue that if $r$ is sufficiently small, then $\theta_1$ does not move by more than $O(1/p)$ with probability $1-p^{-\Omega(1)}$ (and as before, if this high probability event does not occur, the contribution to the risk is negligible). As long as $\theta_1$ does not move by more than $O(1/p)$ while we run DP-SGD, it will remain in $S$ since we assume at the start of DP-SGD, $\theta_1$ is distance $\Omega(\frac{\log p}{\sqrt{p}})$ from the boundary of $S$. Putting it all together, this implies that (i) the excess loss on $\ell_1$ does not increase by more than $O(1/p)$, since $\ell_1$ is $O(1)$-Lipschitz with respect to $\theta_1$, and (ii) since $\theta_1$ stays within $S$ and thus $q(\theta_1) = 1$, the change in $\theta_2$ is the same as the change if we ran DP-SGD on $\ell_2(\theta_2)$ alone. Lemma~\ref{lem:dpsgd-stays-in-ball} implies that $\theta_2$ stays within $B_p(0, 2r)$, and thus DP-SGD on $\ell(\theta_2)$ is the same as running DP-SGD on the purely quadratic loss $\frac{p}{2r^2} \ltwo{\theta_2 - d_2}^2$, with high probability. So by Theorem~\ref{thm:sc-dp-sgd}, DP-SGD with optimal parameters will give $\theta_2$ achieving risk $\tilde{O}(1/p^2)$ on $\ell_2$ alone. This gives an overall risk bound of $\tilde{O}(1/p)$, completing the proof.

Now, the idea is that in DP-SGD as in Theorem~\ref{thm:sc-dp-sgd}, the total movement of $\theta_1$ due to both gradient steps and noise is an increasing function of $r$, so we can set $r$ to be sufficiently small to guarantee $\theta_1$ does not move by more than $O(1/p)$.
Specifically, as long as $\theta_1 \in S$ we have $\ell = \ell_1 + \ell_2$, and so the loss $\ell$ satisfies $\ltwo{\nabla \ell(\theta; d) - \nabla \ell(\theta; d')} = O(p/r)$ for all $d, d'$ (the gradient difference bound on $\ell_2$) as long as $\theta_1 \in S$. We use the optimal setting of parameters in DP-SGD corresponding to $L = O(\frac{p}{r}), m = \frac{p}{r^2}$, noting that the initialization condition in Theorem~\ref{thm:sc-dp-sgd} is satisfied by $\theta_2 = 0$. If we plug these into the parameter settings in Theorem~\ref{thm:sc-dp-sgd}, we get that we should use $T = \Theta(p^2)$ iterations with step-size $\eta_t = \frac{r^2}{p t}$ and per-iteration variance $\sigma = \Theta(\frac{\sqrt{\log(1/\delta)}}{r^2 })$, and achieves risk $\tilde{O}(1/p)$ on $\ell_2$ alone. Then, the movement of $\theta_1$ due to unnoised gradients in DP-SGD is at most $O(1) \cdot \sum_t \eta_t = \Theta(r^2 p \log p)$ (here we use the fact that the Lipschitz constant with respect to $\theta_1$ is $O(1)$ within $S$, not $O(p/r)$), and with high probability the movement due to the noise is $O(\sqrt{p^4} \sqrt{\sum_t \eta_t^2} \sigma) = O(p^{3/2} r \sqrt{\log(1/\delta)})$. So if $r = O(\frac{1}{p^{5/2} \sqrt{\log(1/\delta)}})$, we get the desired upper bound on the movement of $\theta_1$ during DP-SGD.
\end{proof}

\begin{proof}[Proof of Theorem~\ref{thm:quadratic-mismatch}]
The proof is almost exactly the same as Theorem~\ref{thm:quadratic-unconstrained}, so we only highlight the changes to that proof.

We define $\calT$ similarly to  Theorem~\ref{thm:quadratic-unconstrained}: For  each $\tau = \tau_1 \times \tau_2$ in $\calT$ as defined in Theorem~\ref{thm:quadratic-unconstrained}, we replace it with $(\tau_{pub} = \tau_1 \times Z, \tau_{pub} = \tau_1 \times \tau_2)$ where $Z$ is a point distribution on the origin.

Now, the lower bound in on $\calA_{priv}$ in (1) can be proven exactly the same as in Theorem~\ref{thm:quadratic-unconstrained} since we're using the same set of private distributions. The lower bound on $\calA_{pub}$ follows similarly to Theorem~\ref{thm:quadratic-unconstrained}, since we proved it holds for $\tau_1 \times \tau_2$ where $\tau_1$ can be arbitrary. Alternatively, one can note that $\calA_{pub}$ learns no information about $\tau_2$ from the public data, so it cannot do better on $\ell_2$ than outputting a fixed point, which has risk $\Omega(1)$

The upper bound in (2) follows since the algorithm only evaluates gradients on the public data where $q(\theta_1) = 0$, i.e. it never uses the coordinates in the public data that are changed between this theorem and Theorem~\ref{thm:quadratic-unconstrained}.
\end{proof}
\section{Quadratic Example}\label{app:quadratic}

In this section, we show that our construction holds even if the loss function is quadratic, as long as we are okay with using a constrained optimization problem.

\begin{thm}\label{thm:quadratic}
For every integer $p \geq 1$, for $\calC = \calD = B_{p^4}(0, 1) \times B_{p}(0, 1)$ there exists $\ell$ such that:
\begin{enumerate}
    \item[(1)] For $\delta = o(1/p^2)$, any (non-private) algorithm $\calA_{pub} : \calD^p \rightarrow \calC$, and any $(1, \delta)$-DP algorithm $\calA_{priv} : \calD^{p^2} \rightarrow \calC$ there exists $\tau$ such that:
    \[\Expect{D_{priv} \sim \tau^{p^2}}{\Expect{\theta \sim \calA_{priv}(D_{priv})}{\calL(\theta)} - \calL(\theta^*(\tau(\calA_{priv})))} = \Omega(1)\]
    \[\Expect{D_{pub} \sim \tau^p}{\Expect{\theta \sim \calA_{pub}(D_{pub})}{\calL(\theta)} - \calL(\theta^*(\tau(\calA_{pub})))} = \Omega(1)\]
    \item[(2)] For any $\delta \geq 2^{-p}$, there exists an algorithm $\calA_{mixed} : \calD^{p+p^2} \rightarrow \calC$ which runs projected gradient descent on the first $p$ examples, followed by $(1, \delta)$-DP-SGD on the last $p^2$ examples, such that for any $\tau$:
    \[\Expect{D \sim \tau^{p+p^2}}{\Expect{\theta \sim \calA_{mixed}(D)}{\calL(\theta)} - \calL(\theta^*(\tau))} = O(1/p)\]
\end{enumerate}
\end{thm}
\begin{proof}

We will first state $\ell$ and then prove each item in the theorem statement.
We use the following construction: $\calC = \calD = B_{p^4}(0, 1) \times B_{p}(0, r)$. Let $(\theta_1, \theta_2)$ denote an element of $\calC$, with $\theta_1 \in B_{p^4}(0, 1)$ and $\theta_2 \in B_{p}(0, r)$ for $r = O(\frac{1}{p^{5/2} \sqrt{\log(1/\delta)}})$, and similarly with $(d_1, d_2)$ and $\calD$. We let $\ell((\theta_1, \theta_2), (d_1, d_2)) = \frac{1}{2} \ltwo{\theta_1 - d_1}^2 + \frac{p}{2 r^2} \ltwo{\theta_2 - d_2}^2$.

As in the proof of Theorem~\ref{thm:quadratic-unconstrained}, we will show (1) in two parts: for any $\calA_{priv}$, there exists $\tau_1 \in \calT_1$ such that the desired lower bound holds for $\tau_1 \times \tau_2$ for all $\tau_2 \in \calT_2'$, and that for any $\calA_{pub}$  there exists $\tau_2 \in \calT_2'$ such that the desired lower bound holds for $\tau_1 \times \tau_2$ for all $\tau_1 \in \calT_1$. 

\mypar{Proof of (1) for $\calA_{priv}$}  Fix $\calA_{priv}$ and take any distribution $\tau_1$ over $B_{p^4}(0, 1)$. Let $\tau(\tau_1)$ be the distribution over $(d_1, d_2)$ given by sampling $d_1 \sim \tau_1$ and letting $d_2$ be the origin with probability 1. Let $\calL$ refer to the population loss over $\ell$, and let $\calL_1$ refer to the population loss over $\ell_1(\theta_1) := \ltwo{\theta_1 - d_1}^2, d_1 \sim \tau_1$. 

Now, consider an algorithm $\calA_{priv}'$ that takes $p^2$ samples from $\tau_1$, pads them with the origin to get $p^2$ samples from $\tau(\tau_1)$ in the preceding paragraph, runs $\calA_{priv}$ on these samples, and then takes $\theta_1$ from the output of $\calA_{priv}$. Notice that:

\[\Expect{D \sim \tau(\tau_1)^{p^2}}{\Expect{\theta \sim \calA_{priv}(D)}{\calL(\theta)} - \calL(\theta^*(\tau(\tau_1)))}\]
\[\geq \Expect{D \sim \tau(\tau_1)^{p^2}}{\Expect{(\theta_1, \theta_2) \sim \calA_{priv}(D)}{\calL((\theta_1, 0))} - \calL(\theta^*(\tau(\tau_1)))}\]
\[= \Expect{D \sim \tau_1^{p^2}}{\Expect{(\theta_1, \theta_2) \sim \calA_{priv}(D), d_1 \sim \tau_1}{\frac{1}{2}\ltwo{\theta_1 - d_1}^2} - \calL_1(\theta^*(\tau_1))}\]
\[= \Expect{D \sim \tau_1^{p^2}}{\Expect{\theta_1 \sim \calA_{priv}'(D)}{\calL_1(\theta_1)} - \calL_1(\theta^*(\tau_1))}.\]

By \cref{lem:privatesclb}, the final expression is $\Omega(1)$ for some distribution $\tau_1(\calA_{priv})$. In turn, for the corresponding $\tau(\tau_1(\calA_{priv}))$, $\calA_{priv}$ has excess population loss $\Omega(1)$ in expectation as desired. 

\mypar{Proof of (1) for $\calA_{pub}$}  This follows by an argument symmetric to the previous part, except we use \cref{lem:publicsclb} instead of \cref{lem:privatesclb}, and the observation that minimizing $\frac{p}{2r^2} \ltwo{\theta_2 - d_2}^2$ is equivalent to minimizing $\frac{p}{2} \ltwo{\theta_2 - d_2}^2$ over $B_p(0, 1)$. In particular, the lower bound on just $\ltwo{\theta_2 - d_2}^2$ given by \cref{lem:privatesclb} is $\Omega(1/p)$, and the lower bound of $\Omega(1)$ on $\calL$ follows after using the same reduction as in the proof of (1) and taking into account the multiplier $\frac{p}{2}$.

\mypar{Proof of (2)} 
This follows similarly to Theorem~\ref{thm:quadratic-unconstrained}, so we only highlight the high-level proof and major changes here. A single step of projected gradient descent on the public data gets us to the empirical minimizer of $\theta_1$, which achieves excess risk $O(1/p)$ on $\frac{1}{2}\ltwo{\theta_1 - d_1}^2$. Then, since we are using projected gradient descent, we know $\theta_2$ is distance $O(r)$ from the population minimizer of $\theta_2$, so projected DP-SGD on the private data gets to a point which achieves risk $O(1/p)$ on $\frac{p}{2r^2}\ltwo{\theta_2 - d_2^2}$. By a similar argument to  Theorem~\ref{thm:quadratic-unconstrained}, projected DP-SGD does not cause $\theta_1$ to move by more than $O(1/p)$ with high probability if $r = O(\frac{1}{p^{5/2} \sqrt{\log(1/\delta)}})$. 
\end{proof}

If we want to take this same example and make it unconstrained, an issue arises: A single step of gradient step with step size $1$ will cause $\theta_2$ to move by $1/r^2$, which is far larger than the radius of the ball that $\theta_2$ was restricted to in the constrained setting. In turn, the DP-SGD guarantees worsened. We can remedy this by taking smaller step sizes on the public data so that each step is non-expansive, i.e. $\theta_2$ does not leave the ball and the DP-SGD guarantees still hold. However, in order to do so we need to use step sizes where $\eta = O(r^2)$, which means we will need to take $\Omega(1/r^2)$ steps in order to reduce our distance to the minimizing $\theta_1$ by a constant. Since $r$ is being set to a small value, this is a large number of steps. In other words, it is possible to take this example and make it unconstrained, while still satisfying that public-then-private gradient descent achieves the desired excess loss, but the algorithm will not be efficient.
\end{document}